\documentclass{article}

\usepackage[final]{neurips_2025}

\renewcommand{\paragraph}[1]{\noindent \textbf{#1}}
\usepackage[colorlinks=true, urlcolor={blue!70!black}, linkcolor={blue!70!black}, citecolor={blue!70!black}, pdfborder={0 0 0}]{hyperref}

\usepackage[utf8]{inputenc} % allow utf-8 input
\usepackage[T1]{fontenc}    % use 8-bit T1 fonts
\usepackage{hyperref}
\usepackage{url}            % simple URL typesetting
\usepackage{booktabs}       % professional-quality tables
\usepackage{amsfonts}       % blackboard math symbols
\usepackage{nicefrac}       % compact symbols for 1/2, etc.
\usepackage{microtype}      % microtypography

\bibpunct{(}{)}{;}{a}{,}{,}

\usepackage{amsmath}
\usepackage{amsthm}
\usepackage{mathtools}
\usepackage{cleveref}
\usepackage{thm-restate}
\usepackage{paralist}

\usepackage[dvipsnames]{xcolor}
\usepackage[most]{tcolorbox}

\DeclarePairedDelimiter\norm{\lVert}{\rVert}%

\tcbset{
  method/.style={
    colback=RoyalBlue!5,    % Background color
    colframe=RoyalBlue!80,  % Border color
    coltitle=white,    % Title color
    fonttitle=\bfseries,
    boxrule=0.5pt,     % Border thickness
    arc=2mm,           % Rounded corners
    left=6pt,
    right=6pt,
    top=6pt,
    bottom=6pt,
    enhanced,
    sharp corners=south, % flat bottom
    title={Method}   % Default title (can be overridden)
  }
}

\newtcolorbox{methodbox}[2][]{%
  method,
  title={{\hypersetup{citecolor=white} #2}},
}
\crefname{methodbox}{method}{methods}
\Crefname{methodbox}{Method}{Methods}

\tcbset{
  assumption/.style={
    colback=Maroon!5,    % Background color
    colframe=Maroon!80,  % Border color
    coltitle=white,    % Title color
    fonttitle=\bfseries,
    boxrule=0.5pt,     % Border thickness
    arc=2mm,           % Rounded corners
    left=6pt,
    right=6pt,
    top=6pt,
    bottom=6pt,
    enhanced,
    sharp corners=south, % flat bottom
    title={Assumption}   % Default title (can be overridden)
  }
}

\newtcolorbox[auto counter]{assumptionbox}[2][]{%
  assumption,
  title={Assumption~\thetcbcounter: {\hypersetup{citecolor=white} #2}},
}

\newcommand\Bbig{B_\mathrm{big}}
\newcommand\Bsmall{B_\mathrm{small}}
\newcommand\nGbig{\norm{G_\mathrm{big}}^2}
\newcommand\nGsmall{\norm{G_\mathrm{small}}^2}

\newcommand\upd[1]{{\color{orange} #1}}
\renewcommand\upd[1]{#1}  % Uncomment for CR

\title{Critical Batch Size Revisited: A Simple Empirical Approach to Large-Batch Language Model Training}

\author{%
  William Merrill \quad
  Shane Arora \quad
  Dirk Groeneveld \quad
  Hannaneh Hajishirzi \\
  Allen Institute for AI \\
  \texttt{willm@allenai.org}
}

\begin{document}

\maketitle

\begin{abstract}
    The right batch size is important when training language models at scale: a large batch size is necessary for fast training, but a batch size that is \emph{too large} will harm token efficiency. To navigate this tradeoff, \citet{mccandlish2018empirical} suggest that a \emph{critical batch size} (CBS), below which training will not substantially degrade loss, can be estimated based on the gradient noise scale during training. While their method has been adopted in practice, e.g., when training GPT-3, strong assumptions are required to justify gradient noise as a proxy for the CBS, which makes it unclear whether their approach should be trusted in practice, limiting its applicability. In this paper, we introduce a simple, empirical approach to \emph{directly} measure the CBS and show how the CBS evolves over training. Applying our approach to the OLMo models, we find that CBS is near 0 at initialization, increases rapidly at first, and then plateaus as training progresses. Furthermore, we find that this trend holds across different model sizes (1B and 7B), suggesting CBS from small training runs can inform larger-scale training runs. Our findings about how the CBS changes over training motivate \emph{batch size warmup} as a natural way to reliably train language models at large batch size: start the batch size small and increase it as the CBS grows. To validate this claim, we use batch size warmup to train OLMo 1B to slightly better loss than the original training run with 43\% fewer gradient steps. This shows how our framework can be applied to reliably train language models at larger batch sizes, increasing data parallelism without compromising performance.
\end{abstract}

\section{Introduction}

Increasing the throughput of training is important for training large models.
A natural way to increase throughput is by increasing data parallelism, i.e., increasing the \emph{batch size} used during training so that more data can be processed at once and the number of sequential gradient steps can be decreased.
% With multi-GPU training, a major bottleneck in training speed is the communication overhead between GPUs when taking a gradient step.
% Thus, a natural way to increase token efficiency is to reduce the number of gradient steps that need to be made by \emph{increasing the batch size}.
However, naively picking a very large batch size can degrade the performance achieved by a fixed token budget, as larger batches can show diminishing returns in their ability to estimate the population gradient.
Thus, in order to confidently train language models with higher token throughput, it is important to develop theoretical and empirical understanding of large-batch training.

One fundamental concept for large-batch training methodology is the following:

\begin{methodbox}{Critical Batch Size Hypothesis \citep{mccandlish2018empirical}}
There is some critical batch size $B^*$ up to which increasing the batch size (and appropriately modifying the  learning rate) approximately preserves the loss trajectory as a function of tokens trained, but, above which, the loss trajectory degrades.
\end{methodbox}

If such a CBS $B^*$ exists (and we can measure it), it represents a reasonable balance between efficiency and performance (i.e., loss) and is thus a practically useful batch size at which to train.
Working in a simplified theoretical setup, \citet{mccandlish2018empirical} derive a correspondence between the CBS and the \emph{gradient noise scale}, i.e., the variance of the per-example gradients from the training distribution.
They suggest that an estimator for the gradient noise scale should be used in practice as a proxy for the CBS, and this can in turn be used to set the batch size for large-scale pretraining runs.
The noise scale also appears to have been adopted in practice as a proxy for the CBS, having been mentioned explicitly in the GPT-3 technical report \citep{brown2020language}, and inspiring a flurry of methodological innovations for better noise scale estimation \citep{gray2023efficient,gray2024normalization}.

While appealing, the link between the gradient noise and the CBS requires several strong assumptions to justify: specifically, it assumes the SGD optimizer and that gradients are well-conditioned (cf.~\Cref{sec:noise-scale}).
Thus, it is unclear whether the noise scale should be a meaningful proxy for the CBS for language model pretraining in practice, where the Adam optimizer is often used and the optimization may not be well-conditioned.
To this end, we aim to address the following practical questions that remain for effectively leveraging the CBS viewpoint to train language models at larger batch sizes:

\begin{enumerate}
    \item How can we measure the CBS cheaply with minimal assumptions before launching a pretraining run?
    \item How does the CBS change over the course of pretraining and as a function of model size? 
    \item Having measured the CBS, how should we adapt the batch size, learning rate, and other parameters over the course of a pretraining run?
\end{enumerate}

This work documents our attempts to answer these questions in order to operationalize the CBS for large-batch training.
We focus our investigation on the OLMo models \citep{groeneveld-etal-2024-olmo,olmo20252olmo2furious}, due to their open weights and data, making the following contributions:

\begin{enumerate}
    \item First, we introduc an empirical method to directly measure the CBS via \textbf{branched training}. Our method avoids strong assumptions needed to justify the noise scale method from prior work. This lets us trust it more than the noise scale method, which we find unreliable.
    \item We use our method to study how our CBS measurement changes over the course of training, finding it improves rapidly initially but than flattens off. Further, we find that the CBS not depend strongly on model size, in line with past findings using different methodology \citep{zhang2024cbs}.
    \item Our knowledge of the local CBS across training checkpoints suggest a natural \textbf{batch size warmup} strategy for large batch training: begin training with a small batch size and double it whenever the CBS increases sufficiency. We use this strategy to train 1B parameter models with 43\% fewer gradient steps without degrading (and, in fact, slightly improving) final loss.
\end{enumerate}
Overall, our empirical framework for measuring and leveraging the CBS provides simple and principled methodology for improving the efficiency of large-scale training runs and addressing other fundamental questions in the science of language model pretraining.

% \WM{Lots more prior work to cite. See related work section here: \url{https://arxiv.org/pdf/2410.21676}}

\section{Background: Estimating CBS via Gradient Noise Scale} 
\label{sec:noise-scale}

Past empirical work has aimed to measure the CBS by 
launching many training runs to the same target loss, which is expensive \citep{zhang-2019-algorithmic,zhang2024cbs},
or by using the gradient noise scale as a proxy \citep{mccandlish2018empirical,gray2023efficient,gray2024normalization}. 
In contrast, we will introduce a new method that uses a small amount of additional training to estimate the CBS, which is less expensive than launching many full training runs and does not make any strong assumptions like the noise scale.
Before introducing our method, we review the noise scale framework used to estimate the CBS in prior work and the underlying assumptions it relies on.

\citet{mccandlish2018empirical} suggest that the CBS can be measured in terms of the gradient noise scale, i.e., the variance of the gradients within a batch.
Concretely, their recommendations to measure the CBS and adapt the learning rate are as follows:

\begin{methodbox}{Existing Method: Noise Scale Proxy for CBS \citep{mccandlish2018empirical}}
    Let $G$ be the full gradient and let $\Sigma$ be the covariance matrix for the gradient across data examples.
    We first compute $\mathcal B_\mathsf{simple}$ as a proxy for the CBS (using an efficient statistical estimator):
    \begin{equation*}
        \mathcal B_\mathsf{simple} = \frac{\mathrm{tr}(\Sigma)}{\norm{G}^2} . \label{eq:noise-scale}
    \end{equation*}
    We set the modified batch size to $\mathcal B_\mathsf{simple}$ and \emph{linearly} scale the learning rate $\eta^*$:
    \begin{align}
        B^* &= \mathcal B_\mathsf{simple} \\
        \eta^* &= \frac{B^*}{B} \cdot \eta . \label{eq:linear-scaling}
    \end{align}
\end{methodbox}

This viewpoint is attractive due to its simplicity and tractability, and it has inspired improved methods for estimating the noise scale \citep{gray2023efficient,gray2024normalization}.
However, the link between noise scale and CBS requires several strong assumptions to justify, which we will argue motivates revisiting other approaches for measuring the CBS.
\citet{mccandlish2018empirical} consider training a model on a loss surface where the loss landscape is well approximated by its second-order Taylor expansion.
The first crucial assumption in their analysis is that optimizer used to decrease the loss is SGD:

\begin{assumptionbox}{SGD Optimizer \citep{mccandlish2018empirical}}
    The step taken to reduce the loss is a noisy estimate of the true gradient, computed via $B$ samples.
\end{assumptionbox}

This assumption may seem benign, but it is worth noting that it is \emph{not} typically met in practice, as LMs are trained with the Adam optimizer \citep{kingma-2017-adam}.
Moreover, by analyzing training dynamics in terms of stochastic differential equations \citep{li2021validity}, \citet{malladi2022sdes} argue theoretically that \Cref{eq:linear-scaling} is an appropriate scaling rule for SGD, but not for Adam, where a \emph{square-root} scaling rule is more principled.
\upd{The square-root scaling for Adam is also supported by the theoretical analysis of \citet[Equation 4]{li2024surge}.
Thus, when training with Adam, it seems that the linear scaling rule assumed by \citet{mccandlish2018empirical} should not apply.}
% is unclear whether the noise scale should be a valid proxy for CBS, and, if so, how to operationalize the scaling rule for learning rate.

A more fundamental issue is that \citet{mccandlish2018empirical} also require another strong assumption to derive the noise scale method for estimating the CBS.
In general, their noise-scale-based estimate of the CBS has a more complex form than $\mathcal B_\mathrm{simple}$ involving the Hessian $H$:
\begin{equation*}
    \mathcal B_\mathrm{noise} = \frac{\mathrm{tr}(\Sigma H)}{G^\top H G} .
\end{equation*}
In order to justify that $B^*$ can be computed as $\mathcal B_\mathrm{simple}$,
\citet{mccandlish2018empirical} assume:

\begin{assumptionbox}{Well-Conditioned Optimization \citep{mccandlish2018empirical}}
    The Hessian $H$ is a multiple of the identity matrix. It follows that $\mathcal B_\mathrm{noise} = \mathcal B_\mathrm{simple}$.
\end{assumptionbox}

This is a strong assumption required to justify using $\mathcal B_\mathrm{simple}$ as a proxy for the CBS because computing Hessians would be too expensive to be practice.
\citet{mccandlish2018empirical} suggest informally that, without this assumption, $\mathcal B_\mathrm{simple}$ \emph{may} still be correlated with $\mathcal B_\mathrm{noise}$, but it is not obvious why this should be the case.
Even if this is true, it still poses a real problem for the noise scale methodology, since the goal of the method is to produce an \emph{absolute} measure of $B^*$. It is unclear for practitioners what coefficient should be used to translate $\mathcal B_\mathrm{simple}$ to $B^*$---and, more fundamentally, whether it is even valid to assume that such a coefficient exists.

\section{Our Method: Measuring the CBS via Local Branched Training} \label{sec:our-method}

As discussed in \Cref{sec:noise-scale}, using the gradient noise scale $\mathcal B_\textrm{simple}$ as a proxy to estimate the CBS relies on several strong assumptions.
In light of this, it unclear whether we can trust the gradient noise scale as a proxy for CBS.
We thus argue that we should instead aim to measure the CBS \emph{directly}, without the need for any strong assumptions to justify an indirect proxy.
After introducing our measurement approach in this section, we will show in \Cref{sec:batch-size-warmup} how these measurements can be applied to train language models to the same (or better) target loss with fewer gradients steps.
% \hanna{one big question that arises is: why has nobody done this before if we call this method "simple" and we can calculate things directly? I feel like we are underselling the contribution or observation here?}

\subsection{Method}

We introduce a simple \textbf{branched training} method that directly approximates the CBS by launched branched training runs from a checkpoint, which allows us to identify $B^*$ as the largest batch size that does not degrade in loss relative to smaller batch sizes as visualized in \Cref{fig:measure-cbs}. 
To make this tractable, we train only for a fixed token budget $\Delta$, assuming that if $B^*$ recovers in loss by $\Delta$, its loss will continue to match smaller batch sizes onwards as well.
This allows us to estimate the CBS with only a small amount of additional training (controlled by $\Delta$). 
Further, as we will find later in \Cref{fig:measure-cbs}, the CBS trend remains consistent across model sizes, so CBS measurements with small models could be used to inform large-scale training runs.

% \hanna{another idea, can we show it with a figure? like show the delta sliding windoes and somehow show the max k, etc. } \WM{That was my intent with \Cref{fig:loss-curves}, but maybe I can change it to make the full method clearer}
% \hanna{sounds good... can you refer to figure 1 then in the above description? }

\begin{methodbox}{Our Method: Branched Training to Measure CBS} \label{box:empirical-cbs}
    Given a training checkpoint with original batch size $B$ and learning rate schedule $\eta$, we aim to measure the critical batch size $B^*$.
    Let $f(\eta)$ be the learning rate scaling rule: $f(k) = k$ for SGD and $\sqrt{k}$ for Adam.
    We create several training branches with modified batch size $k \cdot B$ and learning rate $f(k) \cdot \eta$ and train for a small number of tokens $\Delta$ to get loss $L_k$---following standard practice, we take $L_k$ to be the smoothed loss.
    We then define $k^*$ as the maximum $k$ such that, for all $k < k^*$, $L_{k^*} \leq L_k + \epsilon$, where $\epsilon$ is a tolerance parameter for ``similar'' losses.
    
    \paragraph{} We  then define the CBS $B^*$ and scaled learning rate $\eta^*$ as
    \begin{align*}
        B^* &= k^* \cdot B \\
        \eta^* &= f(k^*) \cdot \eta .
    \end{align*}
\end{methodbox}

\begin{figure}
    \centering
    \includegraphics[width=0.32\linewidth]{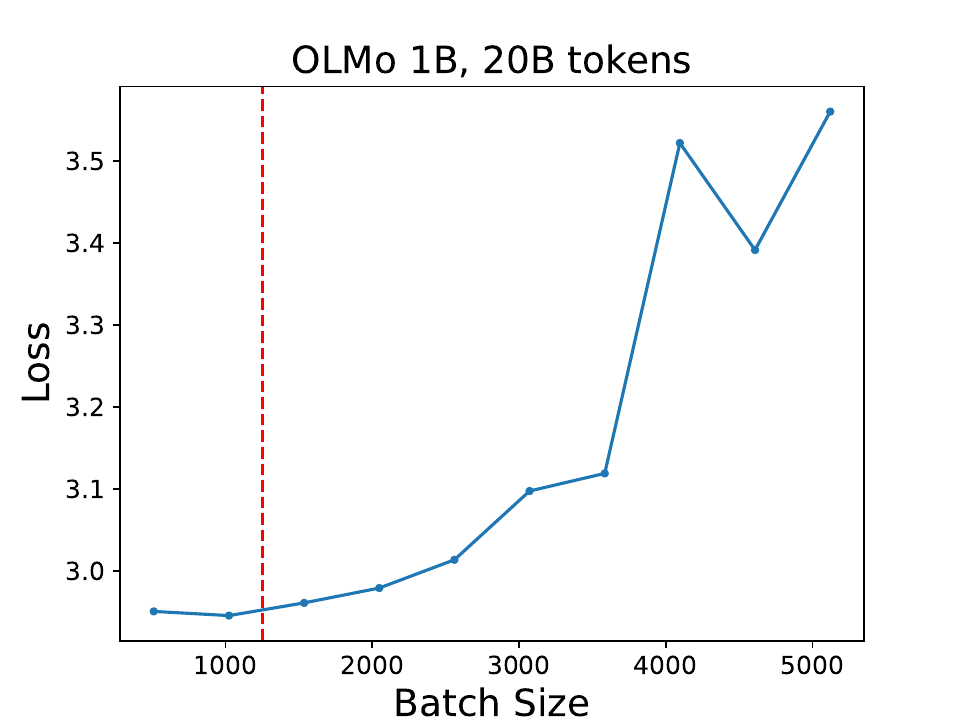}
    \includegraphics[width=0.32\linewidth]{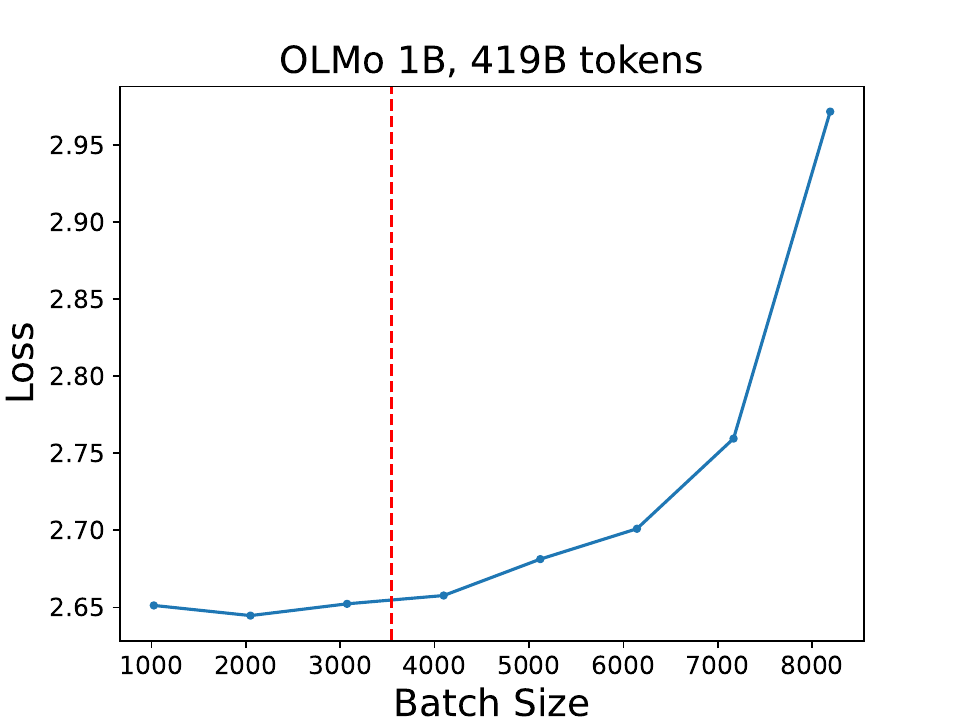}
    \includegraphics[width=0.32\linewidth]{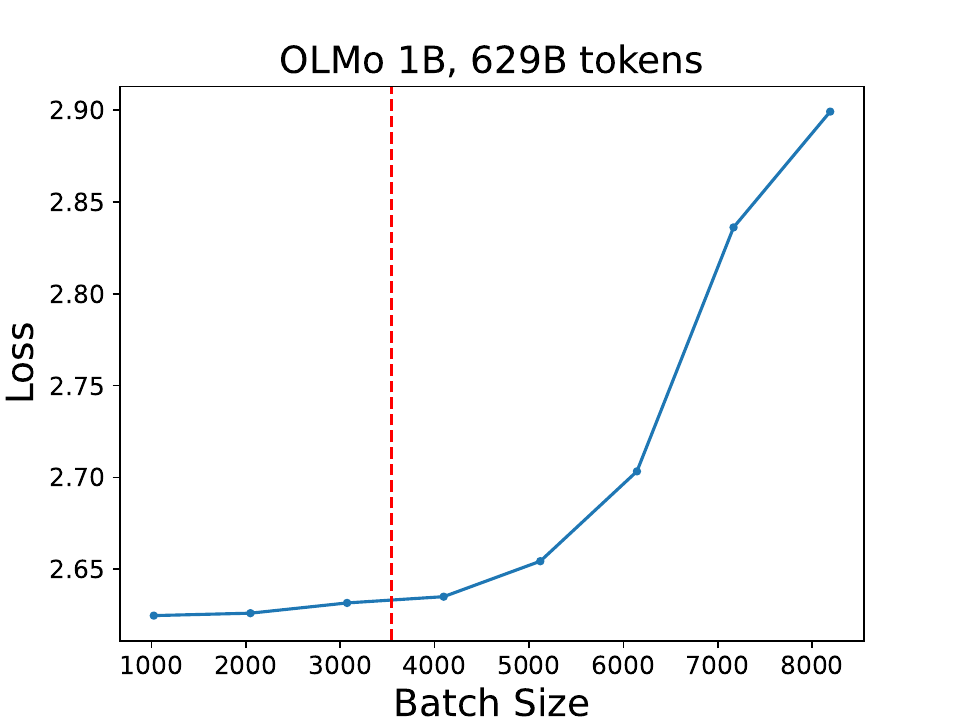}
    \caption{Smoothed final loss after branched training at particular checkpoints, with $B^*$ shown as the dotted red line. Each point represents the loss achieved by a specific branched training run after 2B tokens. Our method detects the point at which loss starts to increase, heuristically tolerating noise within $\epsilon = 0.01$.
    These plots show how this plays out for three particular checkpoints; see \Cref{sec:cbs-measurement-details} for loss curves for all checkpoints.}
    \label{fig:loss-curves}
\end{figure}

Our method will empirically estimate the largest batch size $B^*$ at which the local optimization trajectory recovers roughly to its original loss after $\Delta$ steps.
In contrast to the strong assumptions needed to justify the gradient noise scale, the only (weak) assumption our method relies on is:

% \hanna{can we somehow justify that this assumption is reasonable? do we have some empirical curves to show this? } \WM{can't fully justify without launching full training runs. but, within the windows we're looking in, it does seem that once two losses converge they don't separate}
\begin{assumptionbox}{Local Recovery}
    If the loss achieved by batch size $B^*$ recovers to match the loss with batch size $B < B^*$ after training for $\Delta$ tokens, the loss trajectories will remain the same beyond $\Delta$ as well.
\end{assumptionbox}

\paragraph{Implementation Details.} Our method requires specifying two parameters: the window size $\Delta$ and the loss tolerance $\epsilon$.
We also apply smoothing to the loss to reduce noise.
In more detail:

\begin{compactenum}
    \item \textbf{Window Size.}
    Because the optimizer state must update when the batch size is changed, we expect an immediate spike in the loss after adjusting the batch size.
    The window size $\Delta$ represents the number of steps we are willing to wait for the loss to recover from this bump.
    The CBS measurements could, in principle, depend on $\Delta$, with larger values of $\Delta$ potentially producing larger CBS estimates.
    We set $\Delta$ to 2B tokens, which we take as a small, conversative window size relative to our overall pretraining budget of 600B tokens.

    \item \textbf{Loss Tolerance.} Viewing loss as a function of batch size multiplier $k$ (cf.\Cref{fig:loss-curves}), we need a way to determine whether the loss at $k^*$ has increased relative to all $k<k^*$. We operationalize this with a tolerance parameter $\epsilon$, which we set to 0.01 arbitrarily.
    In principle, tolerance could be set in a more principled way using a statistical test in future work.

    \item \textbf{Loss Smoothing.} Pretraining loss is noisy at the batch level, so, following standard practice, we apply exponentially moving average smoothing with parameter $\alpha$ set to 0.5.
\end{compactenum}

It is worth comparing our method to other work that empirically measures the CBS.
Our method of training for a fixed token budget $\Delta$ and measuring the change in loss can be understood as the dual view to measuring the number of steps required to achieve a target loss, which has been used in prior work \citep{zhang-2019-algorithmic,zhang2024cbs}.
Reformulating the measurement in this way has the nice property that the training budget can be fixed in advance.
Under the local recovery assumption, it also allows us to train for only $\Delta$ tokens, which means we do not have to launch full training runs for each batch size.
Finally, unlike \citet{zhang2024cbs}, we apply our method at various checkpoints throughout training, whereas they apply it only from initialization.
This means we can estimate the \emph{local CBS} at a specific point in training rather than just the \emph{global CBS}.

\subsection{Experimental Setup}

We aim to measure the CBS over the course of model training and the role of model size.
Because our method requires pretraining checkpoints and access to the pretraining data, we use the OLMo 1B and OLMo 7B models for our experiments \citep{olmo20252olmo2furious}, whose pretraining data is openly available \citep{soldaini2024dolma}.
For each model, we take a variety of checkpoints over the course of training, allowing us to assess how the CBS changes over the course of training; see \Cref{sec:cbs-measurement-details} for more details.
We also compute the noise scale across training checkpoints using the estimator proposed by \citet{mccandlish2018empirical} to assess whether it is a valid proxy for the CBS we measure.

We define the interval for the CBS at each checkpoint by choosing $B^*$ (as defined in \Cref{sec:our-method}) as a lower bound, and by picking the least $k > k^*$ as an upper bound. The plotted point represents the geometric mean of these interval endpoints.
We measure the CBS in documents, with a pretraining sequence length of 4096 tokens per document.

\subsection{Results: CBS Over Training and Across Model Sizes}

\begin{figure}
    \centering
    \includegraphics[width=0.48\linewidth]{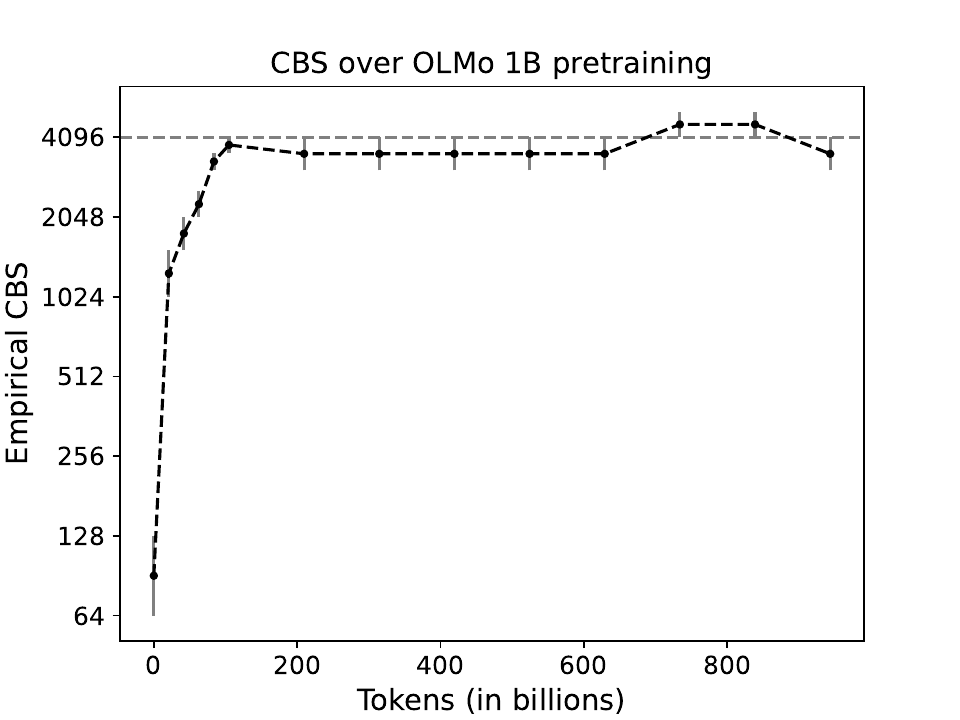}
    \includegraphics[width=0.48\linewidth]{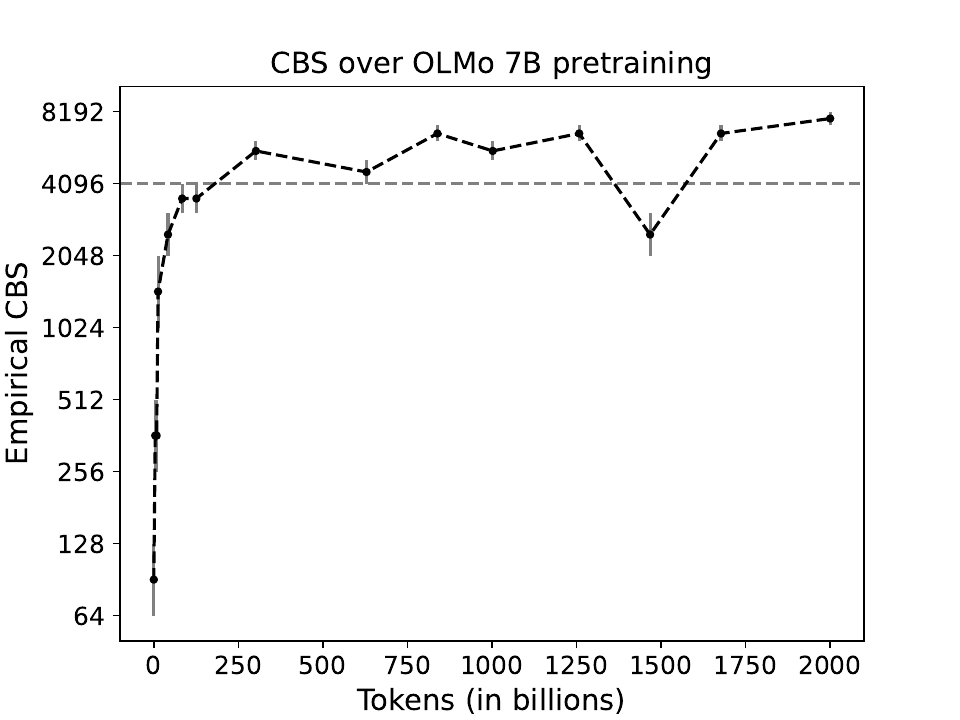}
    \caption{CBS over training for OLMo 1B and 7B, measured in documents (4096 tokens per document). The qualitative trend is similar across both model sizes. The CBS starts near 0, grows rapidly but diminishingly, and plateaus around 4096.}
    \label{fig:measure-cbs}
\end{figure}

\Cref{fig:measure-cbs} shows the CBS $B^*$ measured via \Cref{box:empirical-cbs} across training checkpoints for OLMo 1B and 7B.
The CBS increases over training in a similar way for both model sizes:
the CBS starts near 0, grows rapidly within the first 50k tokens, and then plateaus around 4096.

\paragraph{Impact of Model and Data Size.}
Prior work has suggested that the CBS is largely independent of model size, scaling primarily with data size \citep{zhang2024cbs,bergsma-2025-powerlines}.
This is largely consistent with our findings, as the CBS curves like qualitatively similar at the 1B and 7B scales.
In addition, the fact that the CBS grows over training suggests that the ``aggregate'' CBS should also increase as we train on more data since the average CBS over the course of training should increase.
We elaborate on this in \Cref{sec:scaling-laws}: while CBS growth over training predicts that aggregate CBS should increase with data size, it is unclear whether the CBS growth pattern we observe would predict the aggregate CBS $\propto \sqrt{D}$ scaling law found in prior work \citep{zhang2024cbs,bergsma-2025-powerlines}.

\paragraph{Comparison with Gradient Noise Scale.}
As discussed in \Cref{sec:noise-scale}, the gradient noise scale has been proposed as a proxy to measure the CBS \citep{mccandlish2018empirical}, though this connection relies on strong assumptions to justify.
We thus empirically compare our measurement of the CBS to the gradient noise scale.
\Cref{fig:noise-scale} shows that, for both OLMo 1B and 7B, the gradient noise scale underestimates the CBS by several orders of magnitude. Furthermore, especially for OLMo 7B, the qualitative trend does not match the CBS.
For OLMo 1B, the qualitative pattern is more similar.
However, since this similarity is not found for both models, we conclude that, in general, the noise scale cannot be used reliably as a proxy for the CBS.
% Since our method directly measures noise scale, our interpretation of this discrepancy is that the strong assumptions (\Cref{sec:noise-scale}) required to justify the noise scale proxy are problematic in practice, and thus that the noise scale is not a reliable way to estimate the CBS.

\begin{figure}
    \centering
    \includegraphics[width=0.48\linewidth]{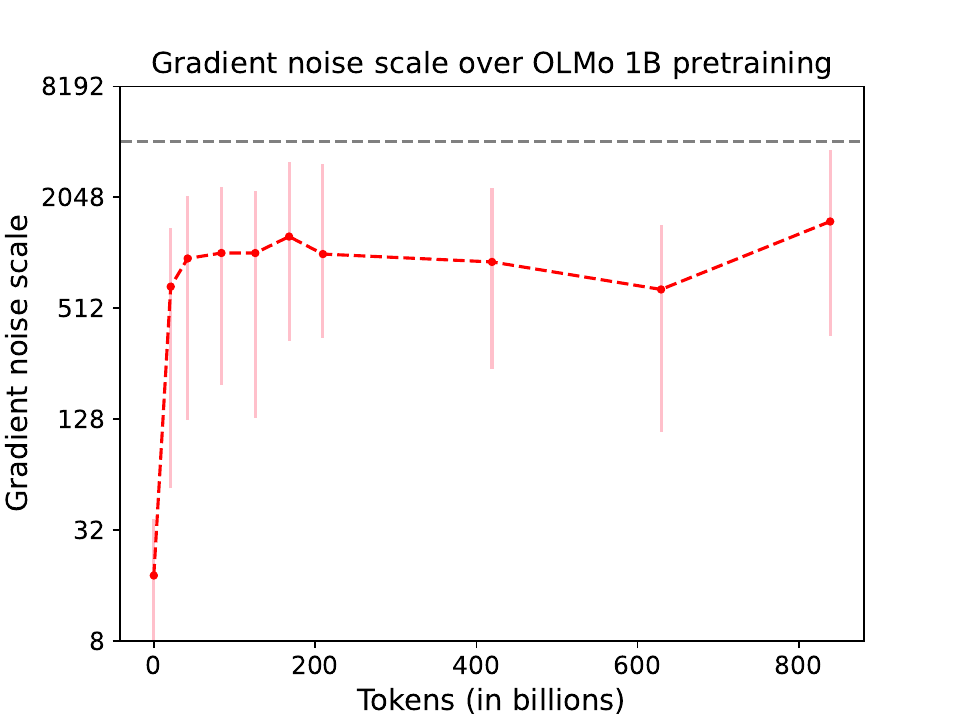}
    \includegraphics[width=0.48\linewidth]{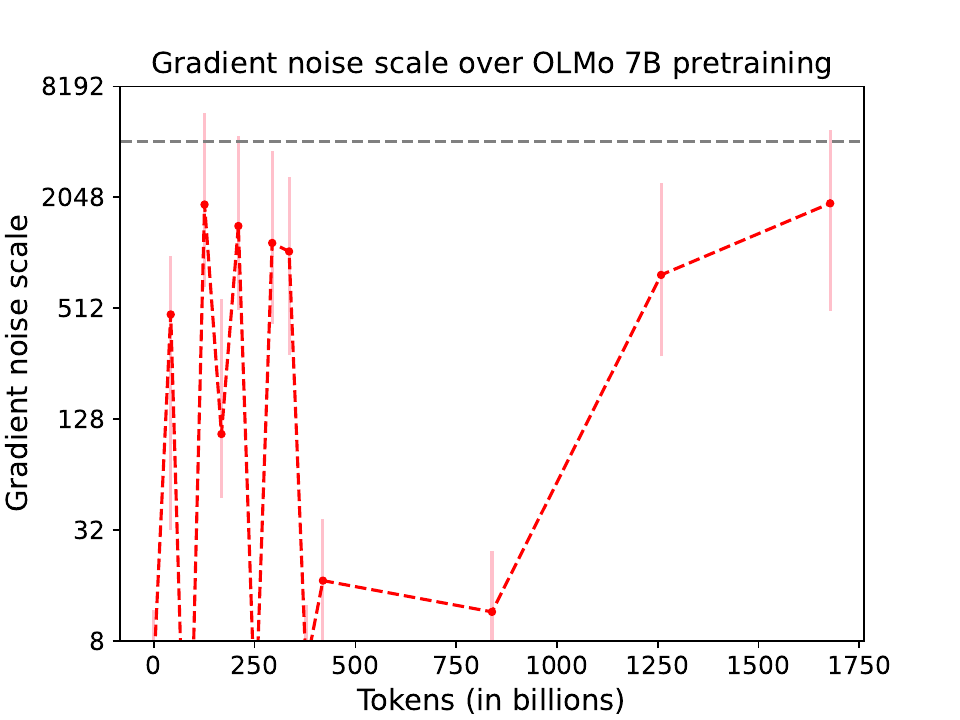}
    \caption{Gradient noise scale for OLMo 1B and 7B computed via the estimator of \citet{mccandlish2018empirical} with 95\% confidence intervals; details in \Cref{app:noise-scale}.
    The gradient noise scale underestimates the CBS (cf.~\Cref{fig:measure-cbs}) and the qualitative trend does not clearly match, especially for OLMo 7B.}
    \label{fig:noise-scale}
\end{figure}

\paragraph{Motivating Batch Size Warmup.}
A central takeaway from these results is that the CBS starts near 0, grows rapidly, and then plateaus.
This suggests that \emph{batch size warmup}, where the batch size is dynamically increased over the beginning of a training run, is a natural way to increase the effective batch size for most of training while avoiding training with a batch size above the CBS.
In the next section, we will discuss our implementation and validation of this idea.

\section{Application: Batch Size Warmup for Larger Batch Training} \label{sec:batch-size-warmup}

A straightforward way to speed up training is to increase the batch size, and, as long as the new batch size is less than the CBS, we can be confident that the loss will be minimized about as effectively as before.
Thus, we can leverage knowledge of the CBS to speed up training.

Furthermore, our local knowledge of the CBS can be leveraged to do this more effectively than if we only had a global sense of the CBS.
If we simply increased a fixed batch size, this would mean that we are training with a batch size that is too large for a short period at the beginning of training, which could potentially destabilize training or degrade final performance. To get around this, we can use our measurement of the local CBS to ``warm up'' the batch size.
We will train at a smaller batch size for the beginning of training when the CBS is small, and then switch to a larger batch size once the CBS grows large enough.
More generally, we can aim to double the batch size whenever we determine the CBS has doubled.
This should allow us to benefit from training at a larger batch size for \emph{most} of training without training at a batch size that is too large at the beginning of training, which could degrade the final loss achieved by the training run.

\subsection{Methodology for Batch Size Warmup} \label{sec:batch-size-warmup-method}

Given an existing training run (at a small batch size), we aim to adapt it to achieve the the same final loss with fewer overall gradient steps (i.e., a larger batch size for most of training).
Assuming an original fixed batch size $B$ and base learning rate $\eta$, our method for training with \emph{batch size warmup} is as follows.
First, we use branched training CBS method (\Cref{sec:our-method}) to measure $B^*_t$, the CBS after training for $t$ tokens, using an existing checkpoint.
When training a new model, we set the batch size and and learning rate as follows:

\begin{methodbox}{Batch Size Warmup} \label{box:batch-size-warmup}
    \begin{enumerate}
        \item At $t = 0$, initialize the batch size to $B_0 = B$ and base learning rate to $\eta_0 = \eta$.
        \item After training for $t$ tokens, if we determine that the CBS exceeds the current batch size ($B^*_t > 2B_t$), we double the current batch size and update the base learning rate following the square-root scaling rule (cf.~\Cref{sec:our-method}):
        \begin{align*}
            B_{t+1} &= 2 B_t \\
            \eta_{t+1} &= \sqrt{2} \eta_t .
        \end{align*}
    \end{enumerate}
\end{methodbox}

This method will ensure that the batch size will increase over training \emph{safely}, i.e., never exceeding the CBS.
Since we found that the CBS increases quite rapidly at the beginning of training, this method will double the batch size twice early in training, reaching a maximum batch size of 4096 by 503B tokens.
This means that we will effectively train at a larger batch size for most of training compared to the original small-batch run.
But, crucially, we can be confident that the batch size will never increase our batch size above the CBS. Thus, we expect the final loss will be comparable to that of the original training across thet training trajectory.

\paragraph{Implementation Details.} The main implementation question is how to operationalize checking $B^*_t > 2B_t$ in order to double the batch size.
In practice, we do this heuristically based on the measurements in \Cref{fig:measure-cbs}: since we only double the batch size twice over training, this involves just picking two thresholds.
In addition, this process could be automated for future training runs using online measurements of the CBS paired with some kind of curve fitting or statistical test.

\paragraph{Design Choices.} As defined above, our method uses a square-root scaling rule, which is well-motivated for Adam \citep{malladi2022sdes}, but in principle a linear scaling rule could also be used.
As batch size warmup only modifies the base learning rate, it is compatible to overlay with existing learning rate schedules (in practice, the models we use follow a cosine schedule).
Finally, we choose to only increase the batch size at powers of two because it is not clear that having a more precise match to the CBS beyond the order-of-magnitude level would be particularly useful and because the OLMo codebase requires the number of GPUs $g$ divides the batch size.
\upd{Thus, doubling is convenient because it easily guarantees the new batch size remains a multiple of $g$.
However, in principle one could update more frequently, e.g., by setting updating the batch size the largest multiple of $g \leq B^*_t$.}
% \DG{OLMo doesn't enforce batch sizes that are powers of two. The problem is that the number of GPUs you have needs to divide the batch size evenly. If I want to write some code that can increase the batch size in a running job (and I don't want to make any assumptions about the
% existing batch size and number of GPUs I have), then doubling is the only thing I can do.}

\paragraph{Connection to Existing Methods.}
Our batch size warmup method is similar to preliminary experiments by \citet[Appendix D]{mccandlish2018empirical} using a dynamic batch size on the SVHN dataset. However, they use their noise scale method (\Cref{sec:noise-scale}) to set the batch size, which is potentially unreliable, and only consider small-scale classification.
In the context of image classification, prior work explored replacing learning rate decay with increasing the batch size \citep{smith-2018-decay}.
This is conceptually related to our batch size warmup approach in that it leverages scaling rules between batch size and learning rate to achieve larger batch training.
However, our method can apply on top of existing learning rate schedules, can generalize to Adam (vs.~SGD), and, most crucially, ensures that the batch size number exceeds the critical batch size.

\subsection{Experimental Setup}
%To showcase the impact of our batch size warmup method, our objective is to evaluate two claims.\hanna{rewrote these two sentences; check if that's accurate?}
%1) Training with our method should match the original small-batch-size run in terms of final loss.
%2) we expect our method to achieve better final loss than just training with a fixed batch size. 
%Therefore, we train a model with batch size warmup and compare it against these other approaches as controls, meaning we have the following experimental conditions:
We evaluate the viability of training with batch size warmup via the following training runs. For all models, we use the default OLMo 1B pretraining hyperparameters unless stated otherwise.

\begin{enumerate}
    \item \textbf{Batch Size Warmup:} We train OLMo 1B with our batch size warmup method as detailed in \Cref{sec:batch-size-warmup-method}.
    We initialize the batch size to 1024.
    Based on a manual reading of the CBS measurements in \Cref{sec:our-method}, we determine that the CBS reaches 2048 by 168B tokens and 4096 by 503B tokens.\footnote{\upd{These thresholds were determined from preliminary CBS measurements at an earlier stage of the project. Our measurements changed slightly after this point, but we deemed that it was not worth it to restart the expensive training runs because we do not think the results should be sensitive to a precise choice of threshold. Going forward, it could be useful to establish systematic methodology for choosing batch size warmup thresholds given the CBS measurements.}}
    Our method thus doubles the batch size at each of these points.
    \item \textbf{Small-Batch Control:} We train OLMo 1B with batch size $B = 1024$ and base learning rate $\eta = \sqrt{2} \cdot 0.0004$.\footnote{The default hyperparameters for OLMo 1B use a batch size of 512 and learning rate of 0.0004. Since we chose an initial base size of 1024 for our experiments, we adapted the learning rate appropriately via the square-root scaling rule \citep{malladi2022sdes}.}
    We expect that our batch size warmup method should be able to achieve similar final loss to the small-batch control with fewer gradient steps.
    \item \textbf{Large-Batch Control:} We train OLMo 1B with a fixed large batch size of $B = 4096$ and base learning rate $\eta = 2\sqrt{2} \cdot 0.0004$.
    Since its batch size will exceed the critical batch size for the initial part of training, we expect that the large-batch control will show degraded loss compared to our batch size warmup method.
\end{enumerate}

\begin{figure}
    \centering
    \includegraphics[width=0.98\linewidth]{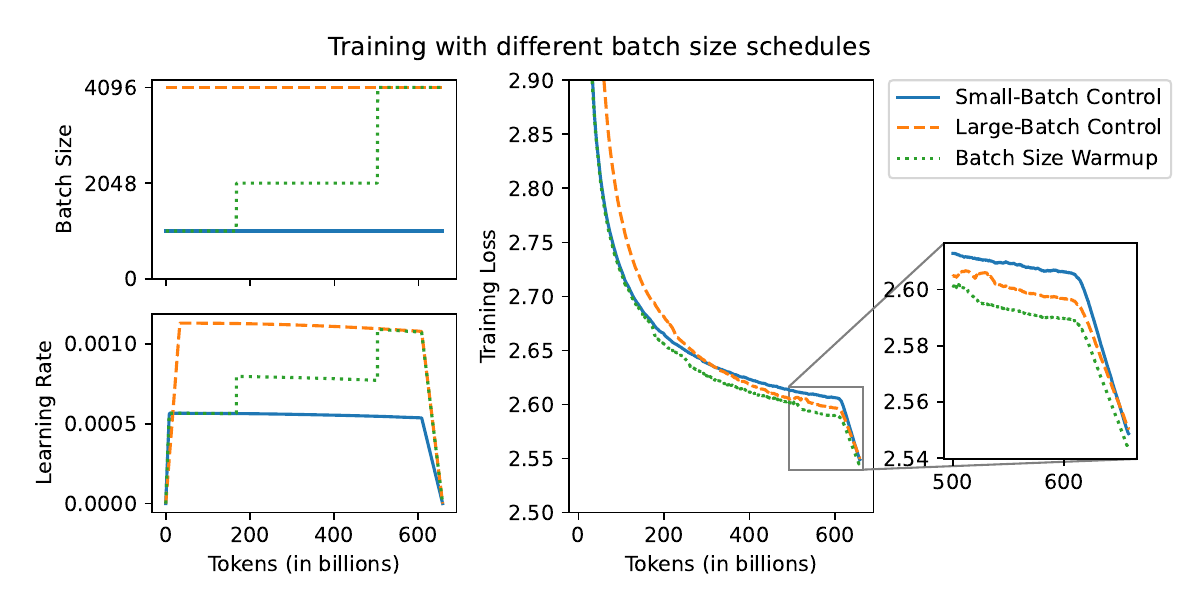}
    \caption{Batch size schedule (left, top), learning rate schedule (left, bottom) and training loss (right) for the pretraining of an OLMo model with different batch size schedules. Training loss is smoothed by taking the moving average over the past 10B tokens.}
    \label{fig:training-diff-methods}
\end{figure}

\begin{table}[!t]
    \centering
    \begin{tabular}{c|cc|c}
        \toprule
        Method & $\downarrow$ PT Loss & $\downarrow$ MT Loss & $\uparrow$ Grad.~Steps Saved \\
        \midrule
        Batch Size Warmup (Ours) & \textbf{2.5891} & \textbf{2.5433} & 43\% \\  % Steps: 1 - 89231/156921
        Small-Batch Control & 2.6057 & 2.5486 & 0\% \\ % Steps: 156921
        Large-Batch Control & 2.5962 & 2.5506 & \textbf{75\%} \\ % Steps: 1 - 39231/156921
        \bottomrule
    \end{tabular}
    \vspace{1em}
    \caption{Loss after pretraining (PT) and mid-training (MT) for each run, averaged over the past 10B tokens and the percentage of gradient steps saved vs.~the small-batch control (including annealing steps).
    Both before and after annealing, batch size warmup slightly outperforms both controls in loss.
    % Before annealing, batch size warmup outperforms both controls in loss, and after annealing, it is comparable in loss to the small-batch control.
    }
    \label{tab:loss-and-time}
\end{table}

\paragraph{Evaluations.} Language model training consists of two steps, pre-training and mid-training \citep{olmo20252olmo2furious}. We evaluate training loss as well as several measures of out-of-distribution loss at the end of the pre-training stage as well as after the mid-training stage.
\begin{compactitem}
    \item \textbf{Loss after Pretraining.} We evaluate the loss at the end of the pre-training stage for all models. 
    The original training run for OLMo 1B ran for 4T tokens, so we do not have the resources to replicate this full training run.
    Instead, we pre-train for 608B tokens using the original learning rate schedule for the longer run.
    % The simplest way to compare the losses of batch size warmup to our controls is at 600B tokens, i.e., partway through pretraining.
    \item \textbf{Loss after Mid-Training.}
    We report loss at the end of mid-training stage, which more closely reflects how these checkpoints are used in language model training. %which involves an additional step of In addition to comparing the loss at these partial checkpoints, we add an additional comparison using \emph{learning rate annealing}, 
    Specifically, starting from final pretraining checkpoint, we linearly anneal the remaining learning rate down to 0 for 50B tokens, keeping the batch size fixed at its final value (\Cref{fig:loss-curves}).
    Because mid-training is a part of the standard language modeling pipeline, we take the loss after mid-training to represent the loss that would be achieved by these training runs in a practical context.
    Furthermore, \citet{olmo20252olmo2furious} suggest that this kind of learning rate annealing can induce significant gains in loss for partial pretraining runs.
    We thus also take the loss after mid-training loss as a proxy for how these runs would compare if we were to fully train them for the full learning rate schedule.

    \item \textbf{Out-of-Distribution Losses.}
    To measure performance beyond the training loss, we also track three kinds of out-of-distribution losses.
    The first two are straightforward cross-entropy losses on common pretraining datasets, C4~\citep{C4} and The Pile~\citep{Pile}.
    The third out-of-distribution loss follows the method from \cite{Bhagia2024EstablishingTS}, which argues for computing the loss (in bits-per-byte; BPB) on the correct answers of multiple popular question-answering datasets such as ARC-Easy, ARC-Challenge, MMLU, etc. For the full list of datasets used, see Appendix \ref{app:bpb}.
\end{compactitem}

\subsection{Batch Size Warmup Results}

\begin{table}[!t]
    \centering
    \begin{tabular}{c|cc|cc|cc}
        \toprule
        Method & \multicolumn{2}{c}{$\downarrow$ Task BPB} & \multicolumn{2}{|c|}{$\downarrow$ C4} & \multicolumn{2}{|c}{$\downarrow$ Pile} \\
        & PT & MT & PT & MT & PT & MT \\
        \midrule
        Batch Size Warmup (Ours) & 1.0316 & 1.0076 & \textbf{2.8049} & \textbf{2.7597} & \textbf{2.1916} & 2.1521 \\
        Small-Batch Control & \textbf{1.0112} & \textbf{0.9999} & 2.8196 & 2.7622 & 2.2073 & \textbf{2.1471} \\
        Large-Batch Control & 1.0571 & 1.01927 & 2.8107 & 2.7658 & 2.1996 & 2.1586 \\
        \bottomrule
    \end{tabular}
    \vspace{1em}
    \caption{According to the method from \cite{Bhagia2024EstablishingTS}, we evaluate downstream performance via BPB on downstream tasks, as well as cross-entropy loss on two held-out sets, C4 and the Pile, both after pretraining and after mid-training.
    Batch-size warmup generally performs comparably or better compared to the small-batch control, suggesting it does not degrade downstream performance.
    }
    \label{tab:downstream}
\end{table}

We compare batch size warmup to the small and large-batch controls at two points: first, after the pretraining stage and, second, after the mid-training stage that anneals the remaining learning rate down to 0.
After pretraining, batch size warmup reaches lower loss than both controls, as shown in \Cref{fig:measure-cbs} and \Cref{tab:loss-and-time}, outperforming the small-batch control in loss by a margin of 0.0166.
After mid-training, batch size warmup still slightly outperforms the small-batch control in loss, this time by a smaller margin of 0.0053.
The large-batch control now achieves the worst overall loss.
Overall, batch size warmup slightly exceeds the small-batch control in loss while using 43\% fewer gradient steps.
In contrast, the large-batch control uses 75\% fewer gradient steps, but its final loss degrades compared to the small-batch control.
Thus, we find batch size warmup is a reliable method to train with fewer gradient steps without degrading final loss.
% These results suggest that batch size warmup is a promising method for large-batch training that not only matches small-batch performance, but, possibly, improves it.

% Training with a larger batch size outperforms the small-batch control.
% \hanna{then second paragraph say: surprisingly, we see that small is worse than large before annealing, and the rest of the points. }

Beyond the final loss, we also evaluate whether batch size warmup leads to similar downstream performance as small-batch training using our downstream-task BPB evaluation as well as validation loss on C4 and the Pile.
% Precisely evaluating the impact of pretraining choices on downstream performance is an unsolved problem, but, following the advice of \citet{Bhagia2024EstablishingTS}, we evaluate downstream performance using bits per byte (BPB) on downstream task data as well as validation loss on C4 \citep{raffel-2020-exploring} and the Pile.
As shown in \Cref{tab:downstream}, we find that BPB on validation sets is broadly competitive with the small-batch control, performing best on three out of four conditions considered.
With the caveat that precisely measuring the downstream impact of pretraining decisions is difficult, this suggests that, just as batch-size warmup does not degrade final loss, it should not degrade downstream measures of performance either.

\section{Conclusion}

In this work, we introduced a simple empirical method for estimating the CBS throughout language model pretraining runs, which can be used to increase batch size (and thus effective token throughput) for large scale training runs without sacrificing performance.
As we discussed, the existing noise scale method \citep{mccandlish2018empirical} for estimating the CBS requires strong assumptions to justify, which our approach avoids.
We used our method to study the evolution of the CBS during training for the OLMo models, finding that CBS increases monotonically but diminishing over the course of training, and that CBS does not seem to depend on model size, in line with prior work \citep{zhang2024cbs}.
Guided by these findings, we showed that our measurements could be used to pick a \textbf{batch size warmup} schedule that enables larger batch training without harming final training loss.
We take these results to demonstrate the validity and utility of our CBS measurement approach, and believe our framework could be useful for increasing the efficiency of future large-scale pretraining efforts.

\upd{There are several details and extensions of our method that would be interesting to explore going forward.
First, it would be interesting to carry out a more systematic analysis of the impact of the hyperparameter $\Delta$ on CBS measurements: how sensitive is the method to $\Delta$ and do different values for $\Delta$ potentially bias the measurement?
It would also be interesting to further investigate the conditions under which noise scale might be a meaningful CBS proxy and used for batch size warmup.
When applying our CBS measurements to picking a batch size warmup schedule, we manually picked doubling threshold based on the CBS.
Going forward, it would be useful to establish a systematic way to set threshold for increasing the batch size given CBS measurements.
There are also other potential methodological improvements, such as removing the arbitrary power of 2 constraint and estimating the CBS in an online fashion.
Overall, these improvements would allow batch size warmup to be applied more robustly and easily across different pretraining setups.}

\section*{Acknowledgments}

We thank Joel Hestness, Sadhika Malladi, and Ananya Harsh Jha for discussions.

\bibliographystyle{abbrvnat}
\bibliography{references}

@misc{mccandlish2018empirical,
    title={An Empirical Model of Large-Batch Training}, 
    author={Sam McCandlish and Jared Kaplan and Dario Amodei and OpenAI Dota Team},
    year={2018},
    eprint={1812.06162},
    archivePrefix={arXiv},
    primaryClass={cs.LG},
    url={https://arxiv.org/abs/1812.06162}, 
}

@inproceedings{malladi2022sdes,
    title={On the {SDE}s and Scaling Rules for Adaptive Gradient Algorithms},
    author={Sadhika Malladi and Kaifeng Lyu and Abhishek Panigrahi and Sanjeev Arora},
    booktitle={Advances in Neural Information Processing Systems},
    editor={Alice H. Oh and Alekh Agarwal and Danielle Belgrave and Kyunghyun Cho},
    year={2022},
    url={https://openreview.net/forum?id=F2mhzjHkQP}
}

@misc{li2021validity,
    title={On the Validity of Modeling SGD with Stochastic Differential Equations (SDEs)}, 
    author={Zhiyuan Li and Sadhika Malladi and Sanjeev Arora},
    year={2021},
    eprint={2102.12470},
    archivePrefix={arXiv},
    primaryClass={cs.LG},
    url={https://arxiv.org/abs/2102.12470}, 
}

@misc{zhang2024cbs,
    title={How Does Critical Batch Size Scale in Pre-training?}, 
    author={Hanlin Zhang and Depen Morwani and Nikhil Vyas and Jingfeng Wu and Difan Zou and Udaya Ghai and Dean Foster and Sham Kakade},
    year={2024},
    eprint={2410.21676},
    archivePrefix={arXiv},
    primaryClass={cs.LG},
    url={https://arxiv.org/abs/2410.21676}, 
}

@misc{olmo20252olmo2furious,
      title={2 OLMo 2 Furious}, 
      author={Team OLMo and Pete Walsh and Luca Soldaini and Dirk Groeneveld and Kyle Lo and Shane Arora and Akshita Bhagia and Yuling Gu and Shengyi Huang and Matt Jordan and Nathan Lambert and Dustin Schwenk and Oyvind Tafjord and Taira Anderson and David Atkinson and Faeze Brahman and Christopher Clark and Pradeep Dasigi and Nouha Dziri and Michal Guerquin and Hamish Ivison and Pang Wei Koh and Jiacheng Liu and Saumya Malik and William Merrill and Lester James V. Miranda and Jacob Morrison and Tyler Murray and Crystal Nam and Valentina Pyatkin and Aman Rangapur and Michael Schmitz and Sam Skjonsberg and David Wadden and Christopher Wilhelm and Michael Wilson and Luke Zettlemoyer and Ali Farhadi and Noah A. Smith and Hannaneh Hajishirzi},
      year={2025},
      eprint={2501.00656},
      archivePrefix={arXiv},
      primaryClass={cs.CL},
      url={https://arxiv.org/abs/2501.00656}, 
}

@misc{soldaini2024dolma,
    title={Dolma: an Open Corpus of Three Trillion Tokens for Language Model Pretraining Research}, 
    author={Luca Soldaini and Rodney Kinney and Akshita Bhagia and Dustin Schwenk and David Atkinson and Russell Authur and Ben Bogin and Khyathi Chandu and Jennifer Dumas and Yanai Elazar and Valentin Hofmann and Ananya Harsh Jha and Sachin Kumar and Li Lucy and Xinxi Lyu and Nathan Lambert and Ian Magnusson and Jacob Morrison and Niklas Muennighoff and Aakanksha Naik and Crystal Nam and Matthew E. Peters and Abhilasha Ravichander and Kyle Richardson and Zejiang Shen and Emma Strubell and Nishant Subramani and Oyvind Tafjord and Pete Walsh and Luke Zettlemoyer and Noah A. Smith and Hannaneh Hajishirzi and Iz Beltagy and Dirk Groeneveld and Jesse Dodge and Kyle Lo},
    year={2024},
    eprint={2402.00159},
    archivePrefix={arXiv},
    primaryClass={cs.CL},
    url={https://arxiv.org/abs/2402.00159}, 
}

@misc{kingma-2017-adam,
    title={Adam: A Method for Stochastic Optimization}, 
    author={Diederik P. Kingma and Jimmy Ba},
    year={2017},
    eprint={1412.6980},
    archivePrefix={arXiv},
    primaryClass={cs.LG},
    url={https://arxiv.org/abs/1412.6980}, 
}

@inproceedings{groeneveld-etal-2024-olmo,
    title = "{OLM}o: Accelerating the Science of Language Models",
    author = "Groeneveld, Dirk  and
      Beltagy, Iz  and
      Walsh, Evan  and
      Bhagia, Akshita  and
      Kinney, Rodney  and
      Tafjord, Oyvind  and
      Jha, Ananya  and
      Ivison, Hamish  and
      Magnusson, Ian  and
      Wang, Yizhong  and
      Arora, Shane  and
      Atkinson, David  and
      Authur, Russell  and
      Chandu, Khyathi  and
      Cohan, Arman  and
      Dumas, Jennifer  and
      Elazar, Yanai  and
      Gu, Yuling  and
      Hessel, Jack  and
      Khot, Tushar  and
      Merrill, William  and
      Morrison, Jacob  and
      Muennighoff, Niklas  and
      Naik, Aakanksha  and
      Nam, Crystal  and
      Peters, Matthew  and
      Pyatkin, Valentina  and
      Ravichander, Abhilasha  and
      Schwenk, Dustin  and
      Shah, Saurabh  and
      Smith, William  and
      Strubell, Emma  and
      Subramani, Nishant  and
      Wortsman, Mitchell  and
      Dasigi, Pradeep  and
      Lambert, Nathan  and
      Richardson, Kyle  and
      Zettlemoyer, Luke  and
      Dodge, Jesse  and
      Lo, Kyle  and
      Soldaini, Luca  and
      Smith, Noah  and
      Hajishirzi, Hannaneh",
    editor = "Ku, Lun-Wei  and
      Martins, Andre  and
      Srikumar, Vivek",
    booktitle = "Proceedings of the 62nd Annual Meeting of the Association for Computational Linguistics (Volume 1: Long Papers)",
    month = aug,
    year = "2024",
    address = "Bangkok, Thailand",
    publisher = "Association for Computational Linguistics",
    url = "https://aclanthology.org/2024.acl-long.841/",
    doi = "10.18653/v1/2024.acl-long.841",
    pages = "15789--15809",
}

@inproceedings{gray2023efficient,
    title={Efficient and Approximate Per-Example Gradient Norms for Gradient Noise Scale},
    author={Gavia Gray and Anshul Samar and Joel Hestness},
    booktitle={Workshop on Advancing Neural Network Training: Computational Efficiency, Scalability, and Resource Optimization (WANT@NeurIPS 2023)},
    year={2023},
    url={https://openreview.net/forum?id=xINTMAvPQA}
}

@inproceedings{smith-2018-decay,
    title={Don't Decay the Learning Rate, Increase the Batch Size},
    author={Samuel L. Smith and Pieter-Jan Kindermans and Quoc V. Le},
    booktitle={International Conference on Learning Representations},
    year={2018},
    url={https://openreview.net/forum?id=B1Yy1BxCZ},
}

@inproceedings{gray2024normalization,
    title={Normalization Layer Per-Example Gradients are Sufficient to Predict Gradient Noise Scale in Transformers},
    author={Gavia Gray and Aman Tiwari and Shane Bergsma and Joel Hestness},
    booktitle={The Thirty-eighth Annual Conference on Neural Information Processing Systems},
    year={2024},
    url={https://openreview.net/forum?id=S7THlpvH8i}
}

@misc{brown2020language,
    title={Language Models are Few-Shot Learners}, 
    author={Tom B. Brown and Benjamin Mann and Nick Ryder and Melanie Subbiah and Jared Kaplan and Prafulla Dhariwal and Arvind Neelakantan and Pranav Shyam and Girish Sastry and Amanda Askell and Sandhini Agarwal and Ariel Herbert-Voss and Gretchen Krueger and Tom Henighan and Rewon Child and Aditya Ramesh and Daniel M. Ziegler and Jeffrey Wu and Clemens Winter and Christopher Hesse and Mark Chen and Eric Sigler and Mateusz Litwin and Scott Gray and Benjamin Chess and Jack Clark and Christopher Berner and Sam McCandlish and Alec Radford and Ilya Sutskever and Dario Amodei},
    year={2020},
    eprint={2005.14165},
    archivePrefix={arXiv},
    primaryClass={cs.CL},
    url={https://arxiv.org/abs/2005.14165}, 
}

@inproceedings{zhang-2019-algorithmic,
    author = {Zhang, Guodong and Li, Lala and Nado, Zachary and Martens, James and Sachdeva, Sushant and Dahl, George and Shallue, Chris and Grosse, Roger B},
    booktitle = {Advances in Neural Information Processing Systems},
    editor = {H. Wallach and H. Larochelle and A. Beygelzimer and F. d\textquotesingle Alch\'{e}-Buc and E. Fox and R. Garnett},
    pages = {},
    publisher = {Curran Associates, Inc.},
    title = {Which Algorithmic Choices Matter at Which Batch Sizes?  Insights From a Noisy Quadratic Model},
    url = {https://proceedings.neurips.cc/paper_files/paper/2019/file/e0eacd983971634327ae1819ea8b6214-Paper.pdf},
    volume = {32},
    year = {2019}
}

@article{Bhagia2024EstablishingTS,
  title={Establishing Task Scaling Laws via Compute-Efficient Model Ladders},
  author={Akshita Bhagia and Jiacheng Liu and Alexander Wettig and David Heineman and Oyvind Tafjord and A. Jha and Luca Soldaini and Noah A. Smith and Dirk Groeneveld and Pang Wei Koh and Jesse Dodge and Hanna Hajishirzi},
  journal={ArXiv},
  year={2024},
  volume={abs/2412.04403},
  url={https://api.semanticscholar.org/CorpusID:274514987}
}

@article{Pile,
  title={The Pile: An 800GB Dataset of Diverse Text for Language Modeling},
  author={Leo Gao and Stella Biderman and Sid Black and Laurence Golding and Travis Hoppe and Charles Foster and Jason Phang and Horace He and Anish Thite and Noa Nabeshima and Shawn Presser and Connor Leahy},
  journal={ArXiv},
  year={2020},
  volume={abs/2101.00027},
  url={https://api.semanticscholar.org/CorpusID:230435736}
}

@inproceedings{C4,
  title={Documenting Large Webtext Corpora: A Case Study on the Colossal Clean Crawled Corpus},
  author={Jesse Dodge and Ana Marasovic and Gabriel Ilharco and Dirk Groeneveld and Margaret Mitchell and Matt Gardner and William Agnew},
  booktitle={Conference on Empirical Methods in Natural Language Processing},
  year={2021},
  url={https://api.semanticscholar.org/CorpusID:237568724}
}

@inproceedings{DataDecide,
  title={DataDecide: How to Predict Best Pretraining Data with Small Experiments},
  author={Ian Magnusson and Nguyen Tai and Ben Bogin and David Heineman and Jena D. Hwang and Luca Soldaini and Akshita Bhagia and Jiacheng Liu and Dirk Groeneveld and Oyvind Tafjord and Noah A. Smith and Pang Wei Koh and Jesse Dodge},
  year={2025},
  url={https://api.semanticscholar.org/CorpusID:277787795}
}

@article{clark2018think,
      title={Think you have Solved Question Answering? {T}ry {ARC}, the {AI2} Reasoning Challenge}, 
      author={Peter Clark and Isaac Cowhey and Oren Etzioni and Tushar Khot and Ashish Sabharwal and Carissa Schoenick and Oyvind Tafjord},
      year={2018},
      eprint={1803.05457},
      archivePrefix={arXiv},
      primaryClass={cs.AI},
      volume={arXiv:1803.05457},
      journal={CoRR}
}

@inproceedings{talmor-etal-2019-commonsenseqa,
    title = "{C}ommonsense{QA}: A Question Answering Challenge Targeting Commonsense Knowledge",
    author = "Talmor, Alon  and
      Herzig, Jonathan  and
      Lourie, Nicholas  and
      Berant, Jonathan",
    editor = "Burstein, Jill  and
      Doran, Christy  and
      Solorio, Thamar",
    booktitle = "Proceedings of the 2019 Conference of the North {A}merican Chapter of the Association for Computational Linguistics: Human Language Technologies, Volume 1 (Long and Short Papers)",
    month = jun,
    year = "2019",
    address = "Minneapolis, Minnesota",
    publisher = "Association for Computational Linguistics",
    url = "https://aclanthology.org/N19-1421",
    doi = "10.18653/v1/N19-1421",
    pages = "4149--4158",
}

@inproceedings{zellers-etal-2019-hellaswag,
    title = "{H}ella{S}wag: Can a Machine Really Finish Your Sentence?",
    author = "Zellers, Rowan  and
      Holtzman, Ari  and
      Bisk, Yonatan  and
      Farhadi, Ali  and
      Choi, Yejin",
    editor = "Korhonen, Anna  and
      Traum, David  and
      M{\`a}rquez, Llu{\'\i}s",
    booktitle = "Proceedings of the 57th Annual Meeting of the Association for Computational Linguistics",
    month = jul,
    year = "2019",
    address = "Florence, Italy",
    publisher = "Association for Computational Linguistics",
    url = "https://aclanthology.org/P19-1472",
    doi = "10.18653/v1/P19-1472",
    pages = "4791--4800",
}

@article{hendryckstest2021,
  title={Measuring Massive Multitask Language Understanding},
  author={Dan Hendrycks and Collin Burns and Steven Basart and Andy Zou and Mantas Mazeika and Dawn Song and Jacob Steinhardt},
  journal={Proceedings of the International Conference on Learning Representations (ICLR)},
  year={2021}
}

@article{Bisk_Zellers_Le_bras_Gao_Choi_2020, title={{PIQA}: Reasoning about Physical Commonsense in Natural Language}, volume={34}, url={https://ojs.aaai.org/index.php/AAAI/article/view/6239}, DOI={10.1609/aaai.v34i05.6239}, number={05}, journal={Proceedings of the AAAI Conference on Artificial Intelligence}, author={Bisk, Yonatan and Zellers, Rowan and Le bras, Ronan and Gao, Jianfeng and Choi, Yejin}, year={2020}, month={Apr.}, pages={7432-7439} }

@inproceedings{sap-etal-2019-social,
    title = "Social {IQ}a: Commonsense Reasoning about Social Interactions",
    author = "Sap, Maarten  and
      Rashkin, Hannah  and
      Chen, Derek  and
      Le Bras, Ronan  and
      Choi, Yejin",
    editor = "Inui, Kentaro  and
      Jiang, Jing  and
      Ng, Vincent  and
      Wan, Xiaojun",
    booktitle = "Proceedings of the 2019 Conference on Empirical Methods in Natural Language Processing and the 9th International Joint Conference on Natural Language Processing (EMNLP-IJCNLP)",
    month = nov,
    year = "2019",
    address = "Hong Kong, China",
    publisher = "Association for Computational Linguistics",
    url = "https://aclanthology.org/D19-1454",
    doi = "10.18653/v1/D19-1454",
    pages = "4463--4473",
}

@article{Sakaguchi_Le_Bras_Bhagavatula_Choi_2020, title={Wino{G}rande: An Adversarial Winograd Schema Challenge at Scale}, volume={34}, url={https://ojs.aaai.org/index.php/AAAI/article/view/6399}, DOI={10.1609/aaai.v34i05.6399}, number={05}, journal={Proceedings of the AAAI Conference on Artificial Intelligence}, author={Sakaguchi, Keisuke and Le Bras, Ronan and Bhagavatula, Chandra and Choi, Yejin}, year={2020}, month={Apr.}, pages={8732-8740} }

@article{cobbe2021training,
  title={Training verifiers to solve math word problems},
  author={Cobbe, Karl and Kosaraju, Vineet and Bavarian, Mohammad and Chen, Mark and Jun, Heewoo and Kaiser, Lukasz and Plappert, Matthias and Tworek, Jerry and Hilton, Jacob and Nakano, Reiichiro and Hesse, Christopher and Schulman, John},
  journal={arXiv preprint arXiv:2110.14168},
  year={2021}
}

@article{lewkowycz2022solving,
  title={Solving quantitative reasoning problems with language models},
  author={Lewkowycz, Aitor and Andreassen, Anders and Dohan, David and Dyer, Ethan and Michalewski, Henryk and Ramasesh, Vinay and Slone, Ambrose and Anil, Cem and Schlag, Imanol and Gutman-Solo, Theo and others},
  journal={arXiv preprint arXiv:2206.14858},
  year={2022}
}

@inproceedings{Wiegreffe2024AnswerAA,
  title={Answer, Assemble, Ace: Understanding How LMs Answer Multiple Choice Questions},
  author={Sarah Wiegreffe and Oyvind Tafjord and Yonatan Belinkov and Hanna Hajishirzi and Ashish Sabharwal},
  booktitle={International Conference on Learning Representations},
  year={2024},
  url={https://api.semanticscholar.org/CorpusID:276903925}
}

@article{austin2021program,
  title={Program synthesis with large language models},
  author={Austin, Jacob and Odena, Augustus and Nye, Maxwell and Bosma, Maarten and Michalewski, Henryk and Dohan, David and Jiang, Ellen and Cai, Carrie and Terry, Michael and Le, Quoc and Sutton, Charles},
  journal={arXiv preprint arXiv:2108.07732},
  year={2021}
}

@article{chen2021evaluating,
  title={Evaluating large language models trained on code},
  author={Chen, Mark and Tworek, Jerry and Jun, Heewoo and Yuan, Qiming and Ponde de Oliveira Pinto, Henrique and Kaplan, Jared and Edwards, Harri and Burda, Yuri and Joseph, Nicholas and Brockman, Greg and others},
  journal={arXiv preprint arXiv:2107.03374},
  year={2021}
}

@misc{bergsma-2025-powerlines,
    title={Power Lines: Scaling Laws for Weight Decay and Batch Size in LLM Pre-training}, 
    author={Shane Bergsma and Nolan Dey and Gurpreet Gosal and Gavia Gray and Daria Soboleva and Joel Hestness},
    year={2025},
    eprint={2505.13738},
    archivePrefix={arXiv},
    primaryClass={cs.LG},
    url={https://arxiv.org/abs/2505.13738}, 
}

@inproceedings{li2024surge,
    title={Surge Phenomenon in Optimal Learning Rate and Batch Size Scaling},
    author={Shuaipeng Li and Penghao Zhao and Hailin Zhang and Samm Sun and Hao Wu and Dian Jiao and Weiyan Wang and Chengjun Liu and Zheng Fang and Jinbao Xue and Yangyu Tao and Bin CUI and Di Wang},
    booktitle={The Thirty-eighth Annual Conference on Neural Information Processing Systems},
    year={2024},
    url={https://openreview.net/forum?id=hD9TUV4xdz}
}

\clearpage
\appendix
\section{CBS Measurement Details} \label{sec:cbs-measurement-details}

Our empirical method for measuring the CBS \Cref{sec:our-method} is, in principle, sensitive to the choice of checkpoints and batch size multipliers.
We therefore document the checkpoints and multipliers we used here.

\paragraph{OLMo 1B.}
When measuring the OLMo 1B CBS with branching, we set the base batch size to 1024 and the base learning rate to $0.0004 \cdot \sqrt{2}$, reflecting the default batch size of 512 and learning rate of 0.0004 in the OLMo codebase under a square-root scaling rule \citep{malladi2022sdes}.
We then chose the following checkpoints and multipliers $k$:
\begin{enumerate}
    \item Step 0: $k$ ranging over 0.0625, 0.125, 0.25.
    \item Steps 10K, 20K, $\ldots$, 50K: $k$ ranging over 0.5, 1, $\ldots$, 5.
    \item Steps 100K, 150K, $\dots$, 450K: $k$ ranging over 1, 2, $\ldots$, 8.
\end{enumerate}
\Cref{fig:all-losses-1b} shows loss vs.~batch size plots for all checkpoints of OLMo 1B.

\paragraph{OLMo 7B.} We set the base batch size to 1024 and the base learning rate to $0.0003$, as specified in the OLMo codebase.
We then chose the following checkpoints and multipliers $k$:
\begin{enumerate}
    \item Step 0: $k$ ranging over 0.0625, 0.125, 0.25.
    \item Steps 1K, 2K, 3K: $k$ ranging over 0.25, 0.5, 1, 2, 3, 4.
    \item Steps 10K, 20K, 30K: $k$ ranging over 1, 2, 3, 4, 5.
    \item Steps 72K, 150K, 200K, 239K, 300K, 350K, 400K, 477K: $k$ ranging over 1, 2, 3, 4, 5, 6, 7, 8.
\end{enumerate}
\Cref{fig:all-losses-7b} shows loss vs.~batch size plots for all checkpoints of OLMo 7B.
These checkpoints were chosen manually as we developed this project.
Over the course of our experimentation, we launched many additional runs beyond the ones discussed above.
Since the choice of $k$ can influence the conclusions of our method, we filtered down the included runs to make the choice of $k$ systematic.

1B runs were launched on a single node of H100 GPUs, and 7B runs were launched on 8 nodes.

\begin{figure}
    \centering
    \includegraphics[width=0.32\linewidth]{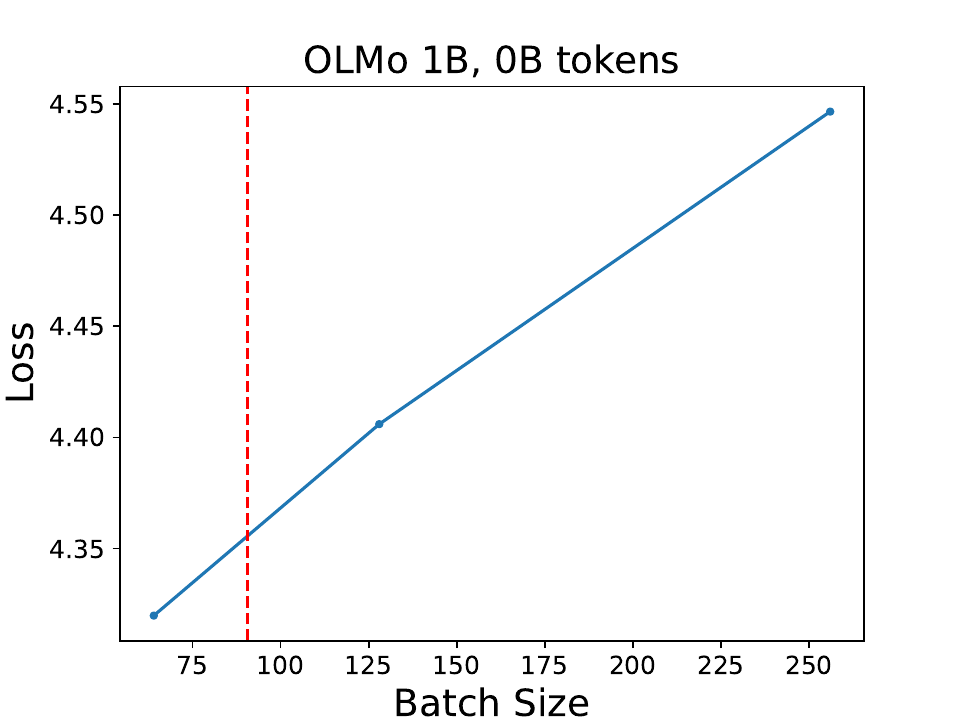}
    \includegraphics[width=0.32\linewidth]{figures/measure-cbs/peteish1-fixlr/from10000.pdf}
    \includegraphics[width=0.32\linewidth]{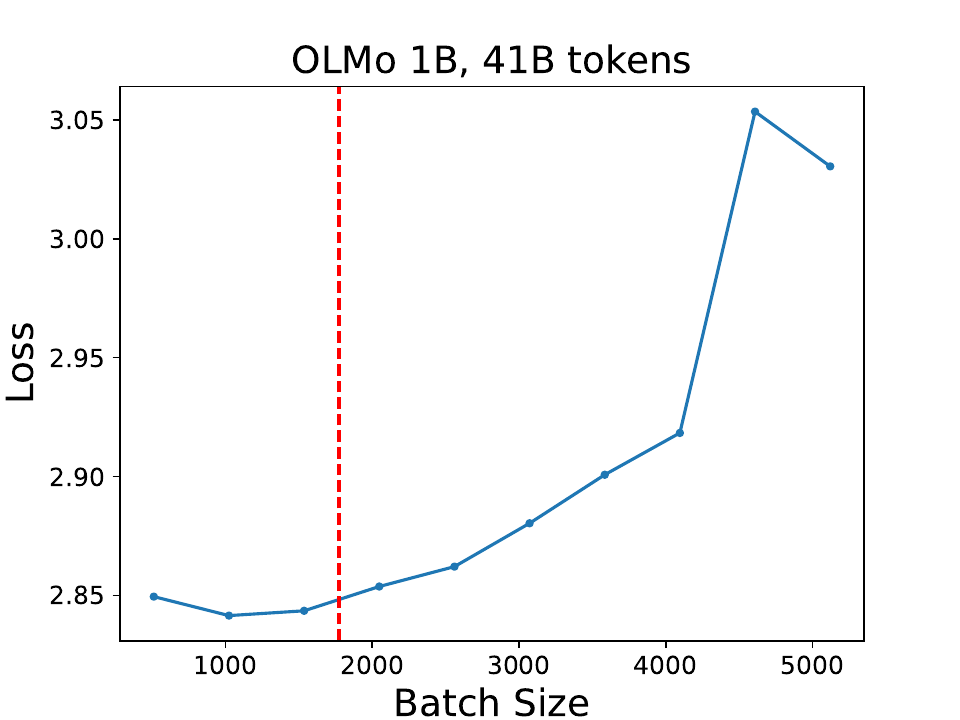}
    \includegraphics[width=0.32\linewidth]{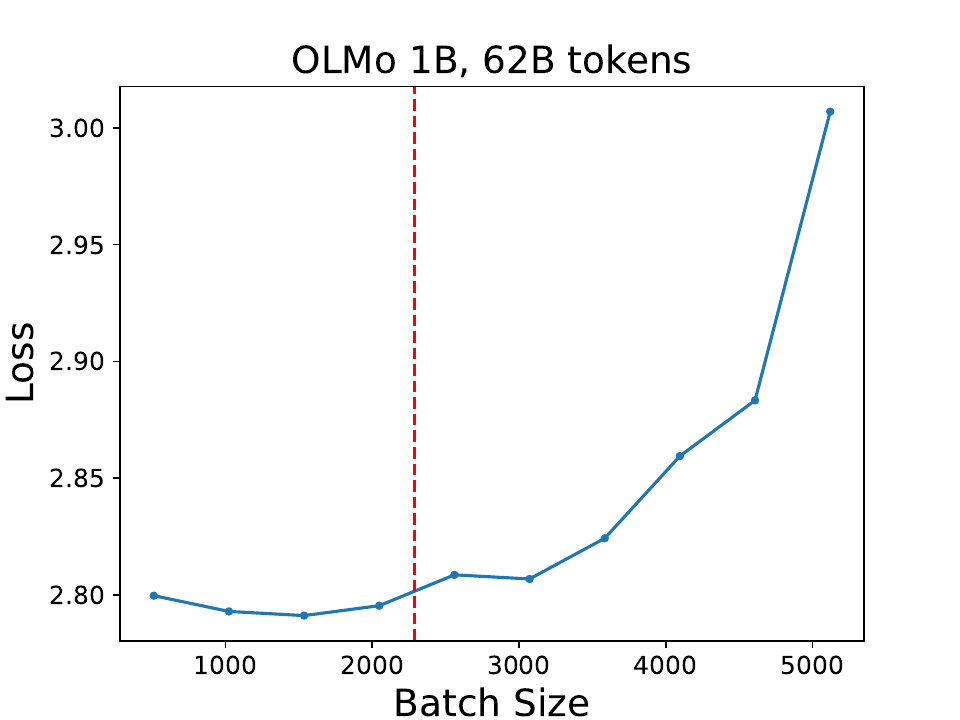}
    \includegraphics[width=0.32\linewidth]{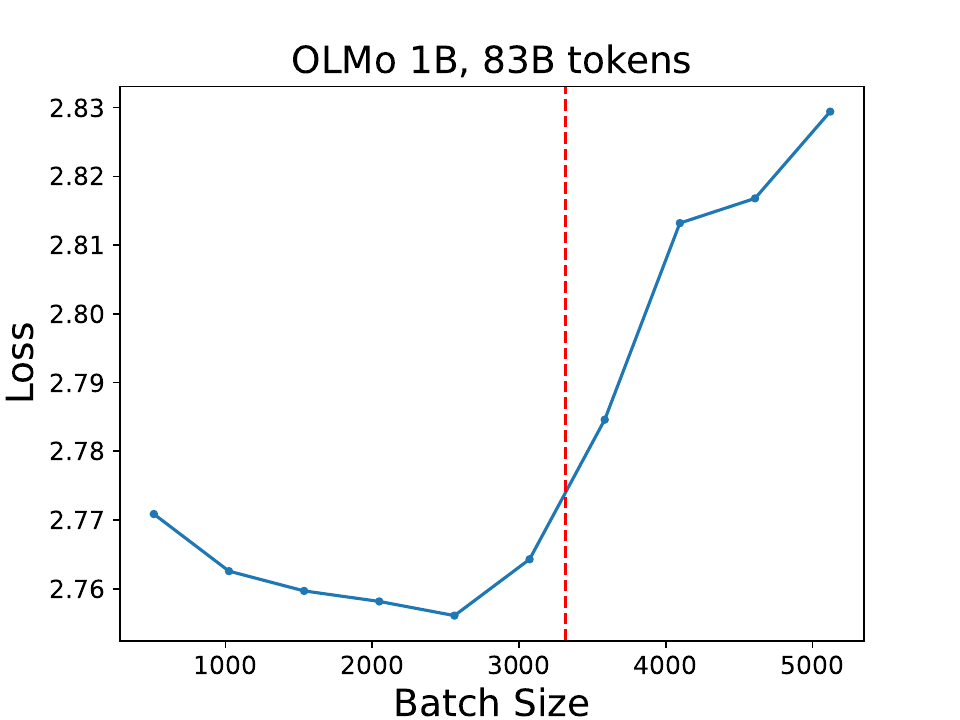}
    \includegraphics[width=0.32\linewidth]{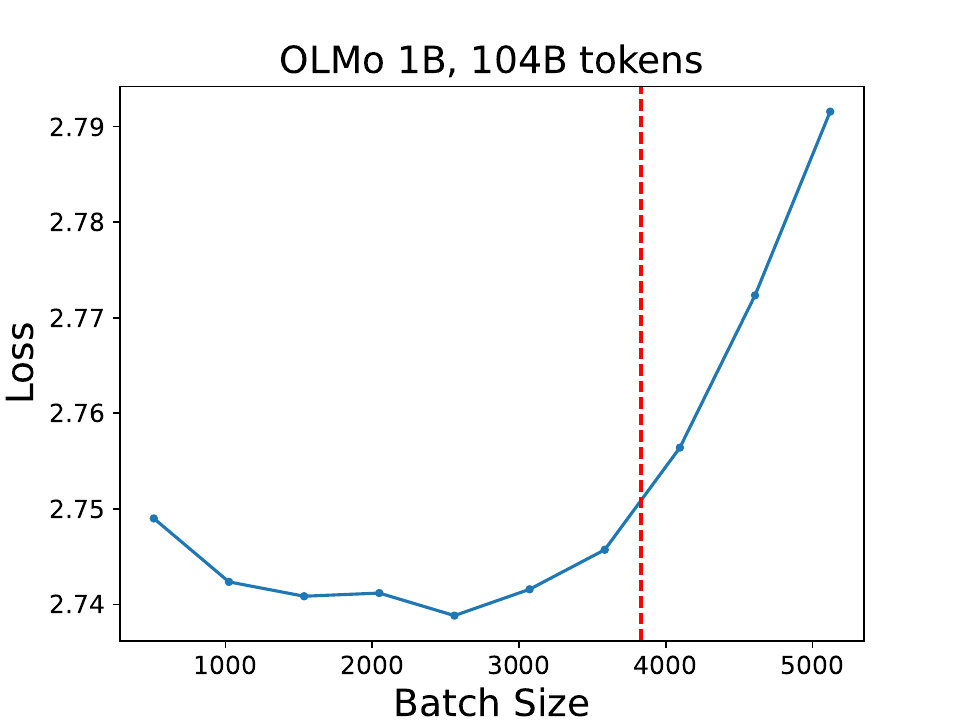}
    \includegraphics[width=0.32\linewidth]{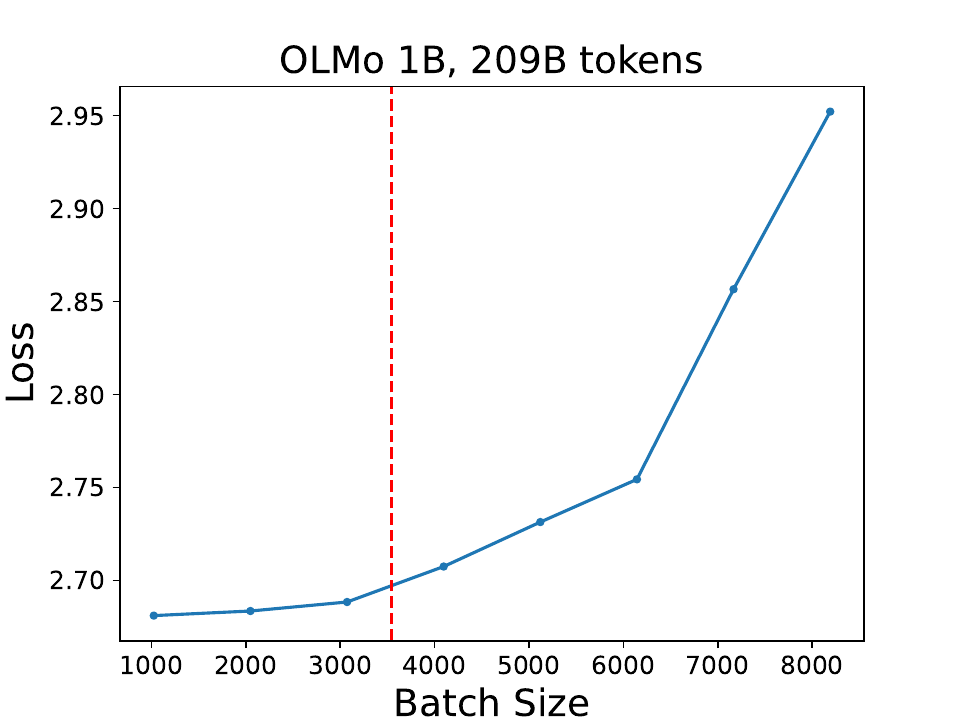}
    \includegraphics[width=0.32\linewidth]{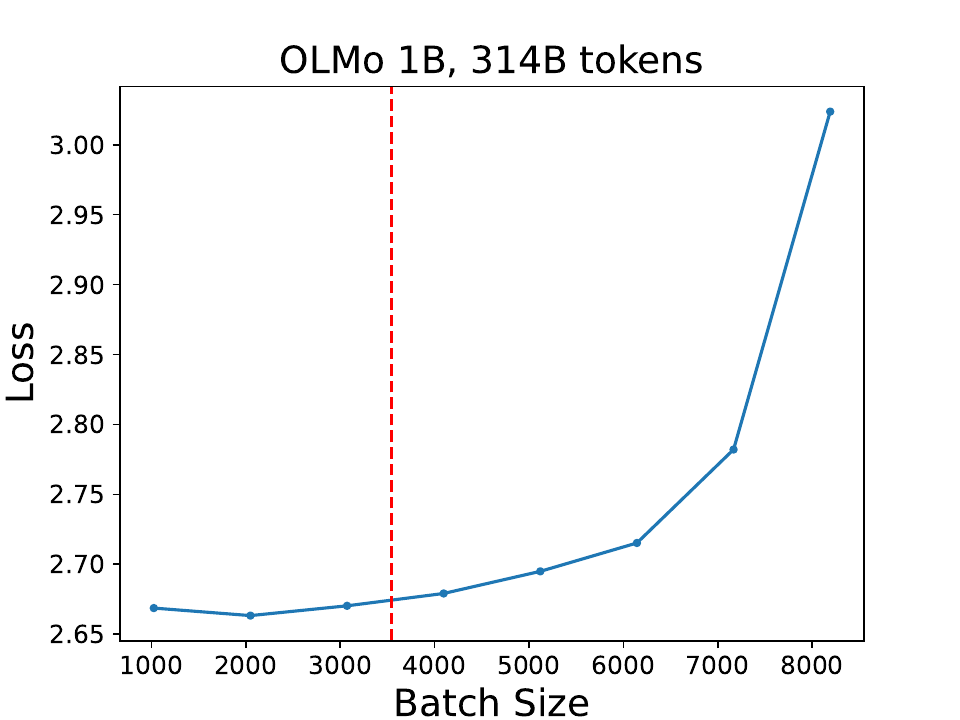}
    \includegraphics[width=0.32\linewidth]{figures/measure-cbs/peteish1-fixlr/from200000.pdf}
    \includegraphics[width=0.32\linewidth]{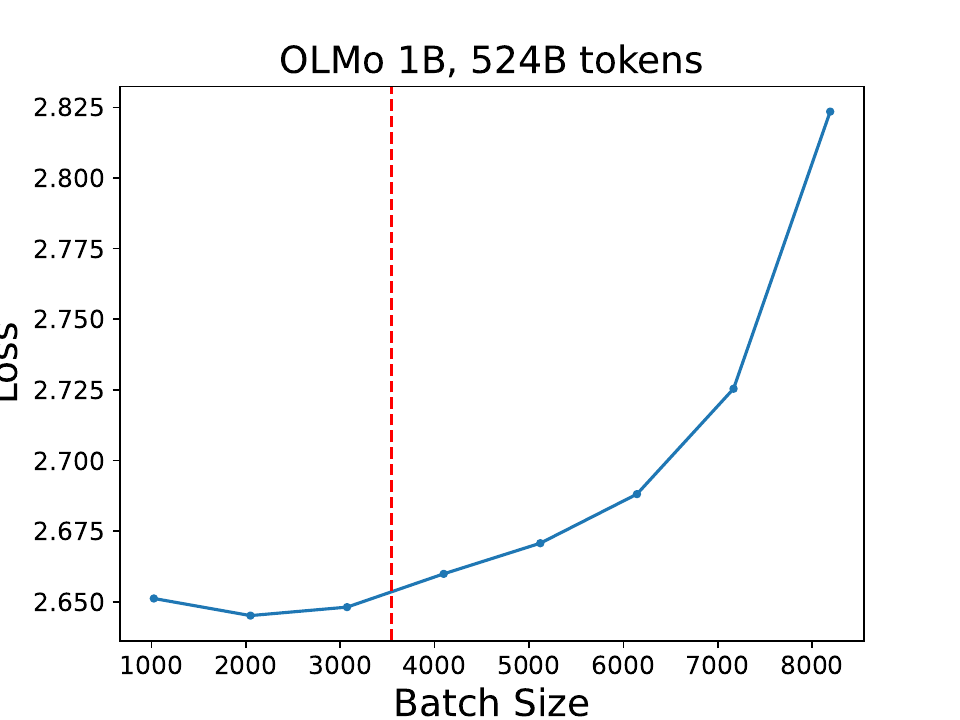}
    \includegraphics[width=0.32\linewidth]{figures/measure-cbs/peteish1-fixlr/from300000.pdf}
    \includegraphics[width=0.32\linewidth]{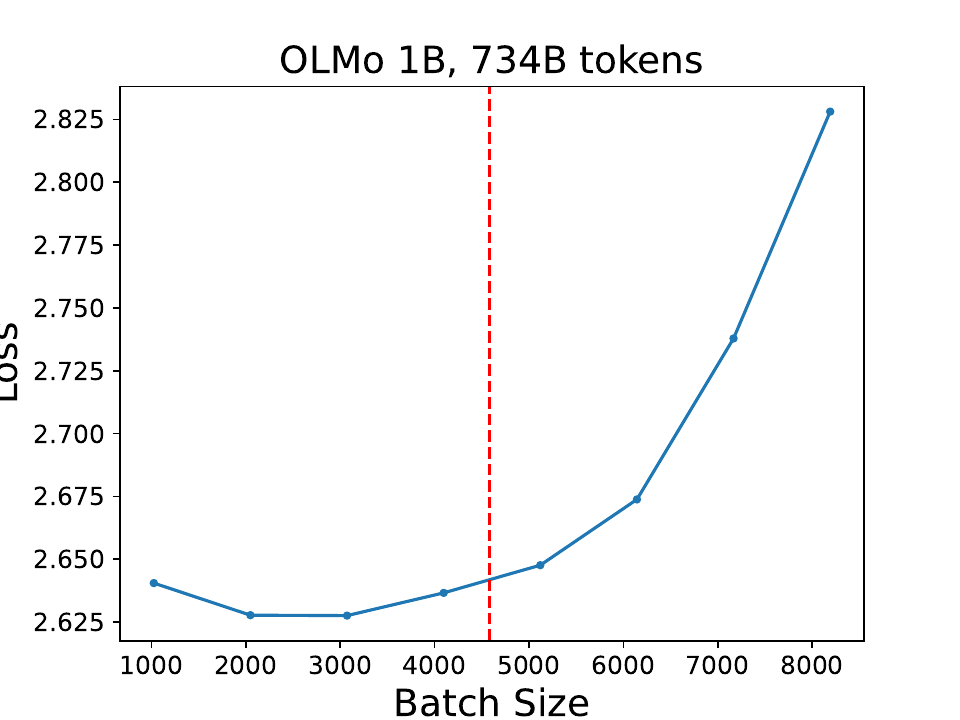}
    \includegraphics[width=0.32\linewidth]{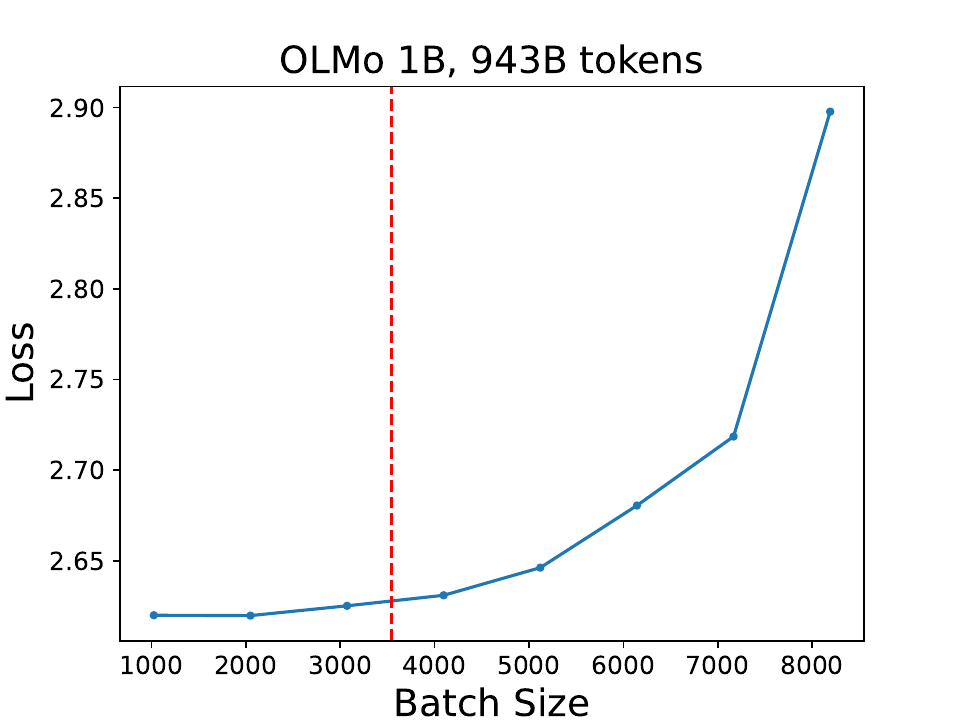}
    \includegraphics[width=0.32\linewidth]{figures/measure-cbs/peteish1-fixlr/from450000.pdf}
    \caption{All loss vs. batch size plots for OLMo 1B. Overall, the red line moves to the right over time, showing that the CBS increases.}
    \label{fig:all-losses-1b}
\end{figure}

\begin{figure}
    \centering
    \includegraphics[width=0.32\linewidth]{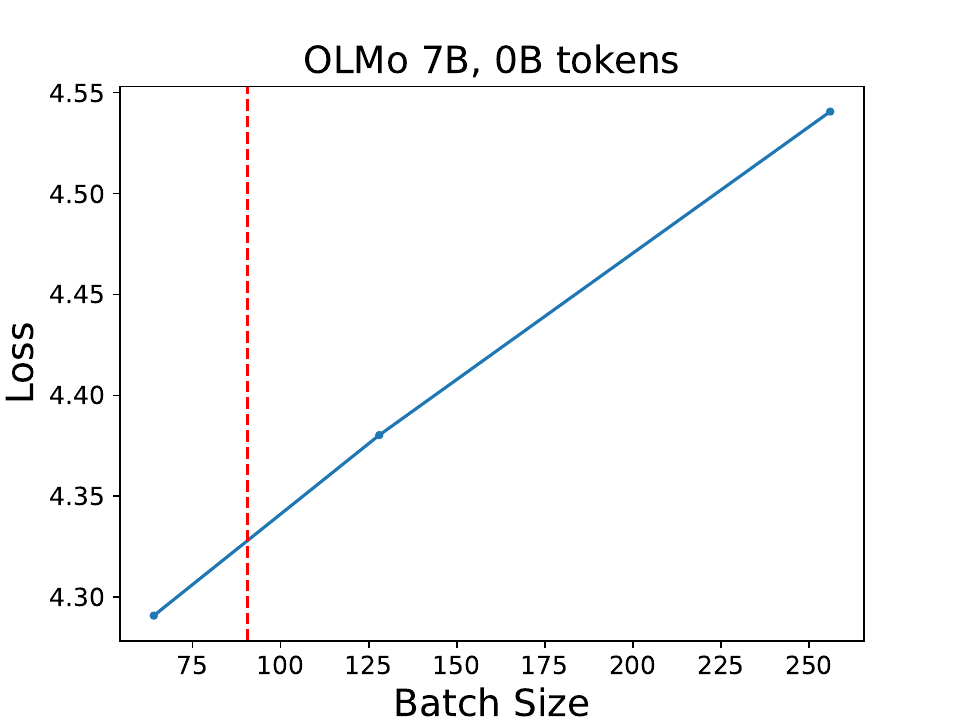}
    \includegraphics[width=0.32\linewidth]{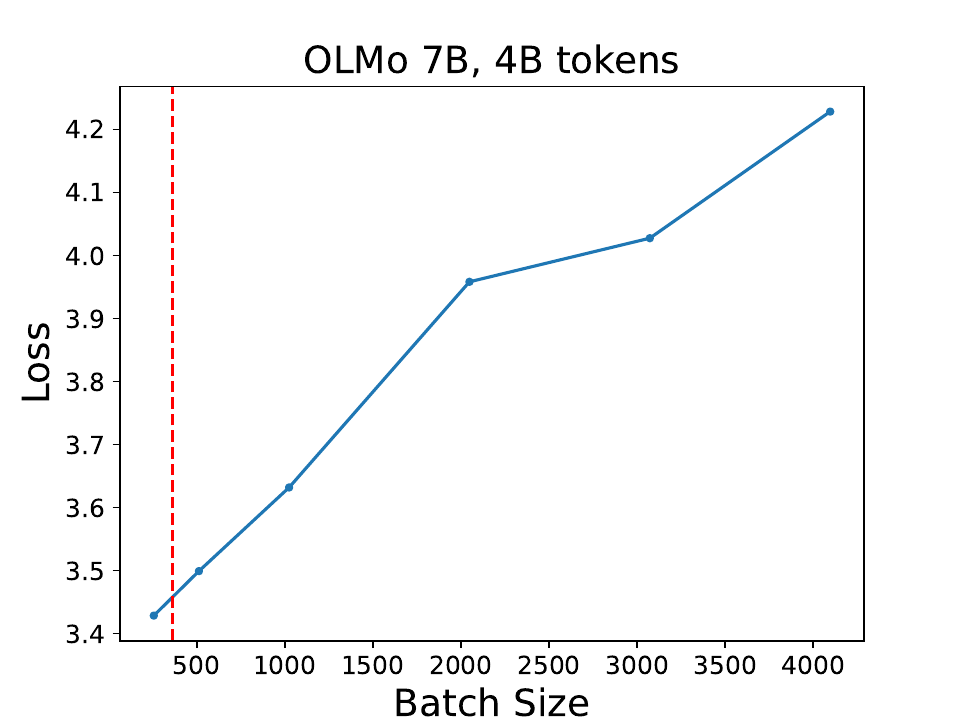}
    \includegraphics[width=0.32\linewidth]{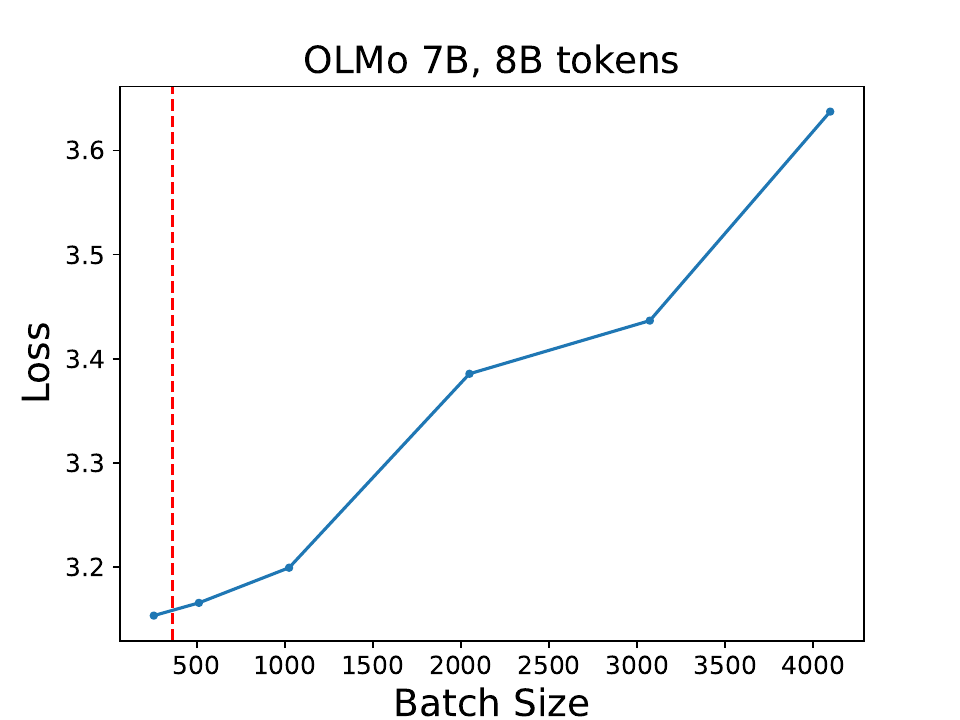}
    \includegraphics[width=0.32\linewidth]{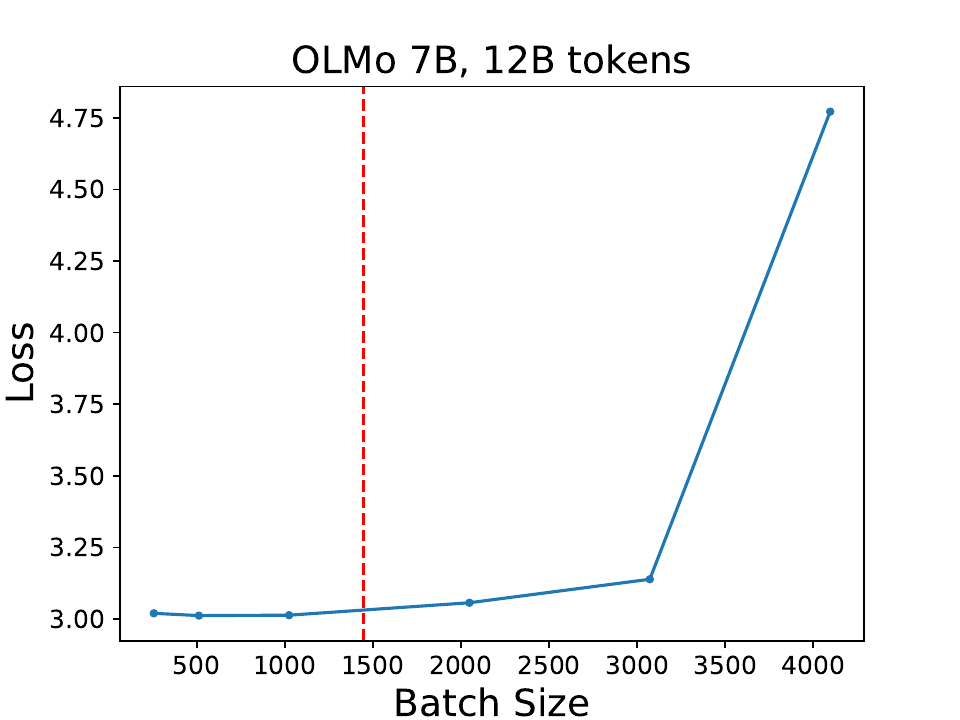}
    \includegraphics[width=0.32\linewidth]{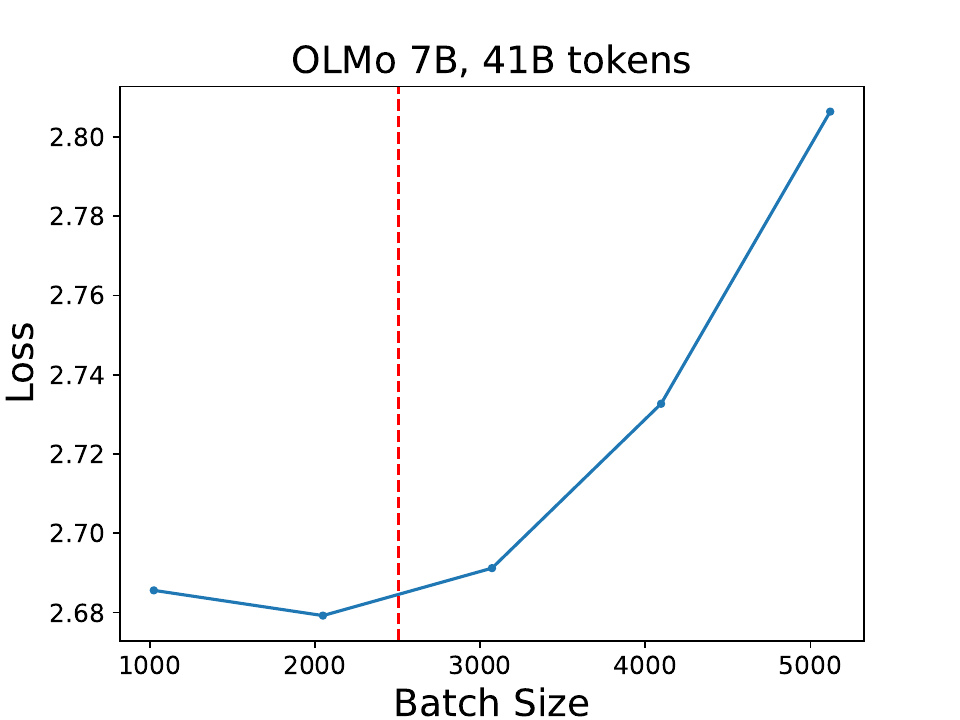}
    \includegraphics[width=0.32\linewidth]{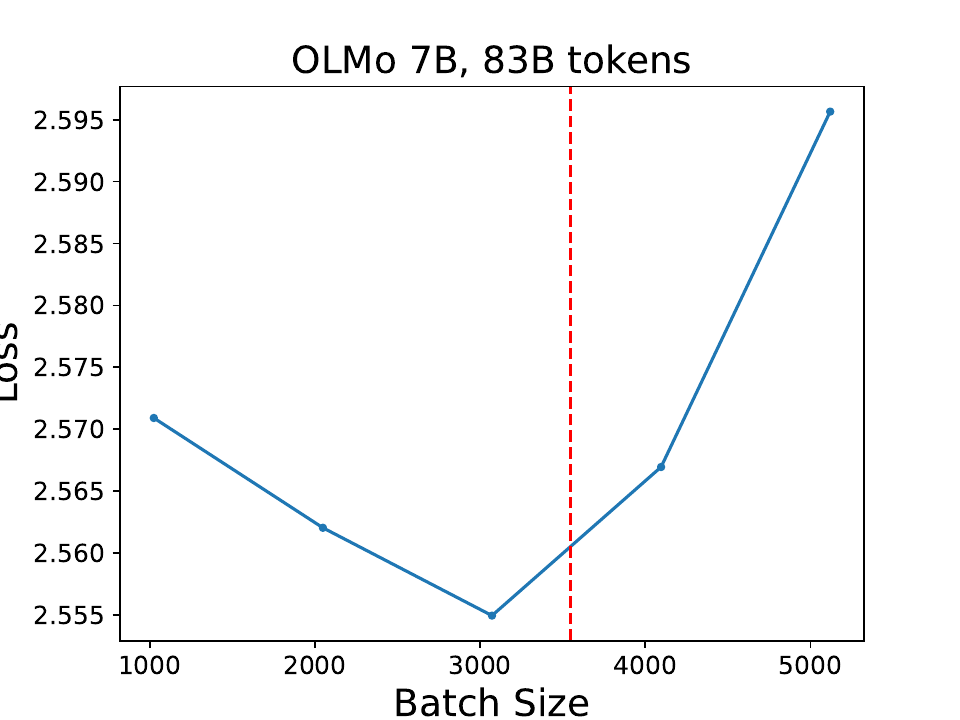}
    \includegraphics[width=0.32\linewidth]{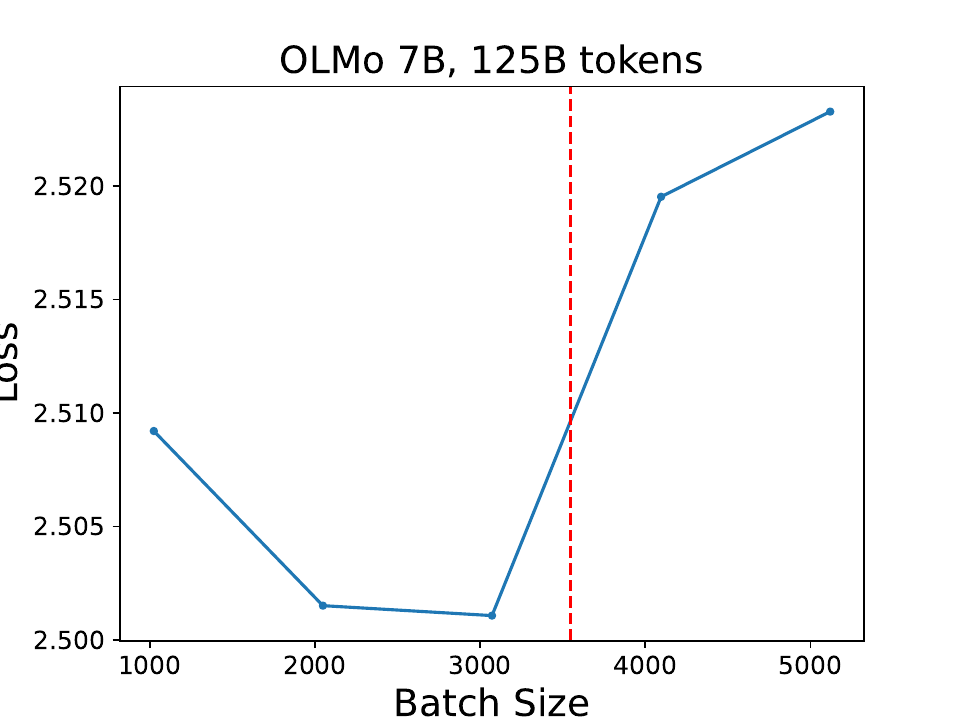}
    \includegraphics[width=0.32\linewidth]{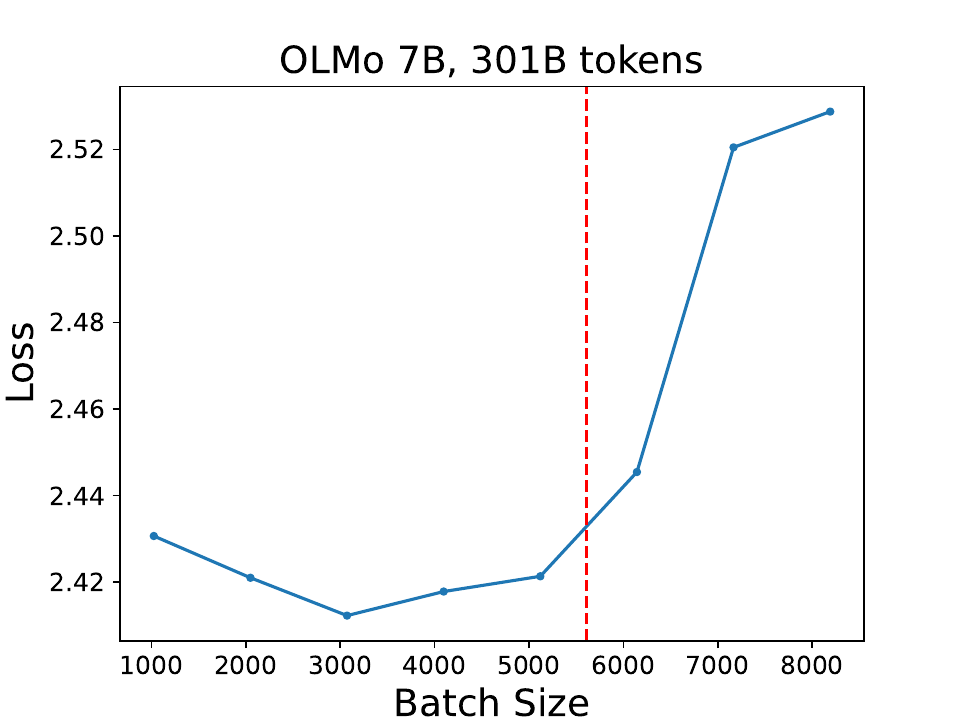}
    \includegraphics[width=0.32\linewidth]{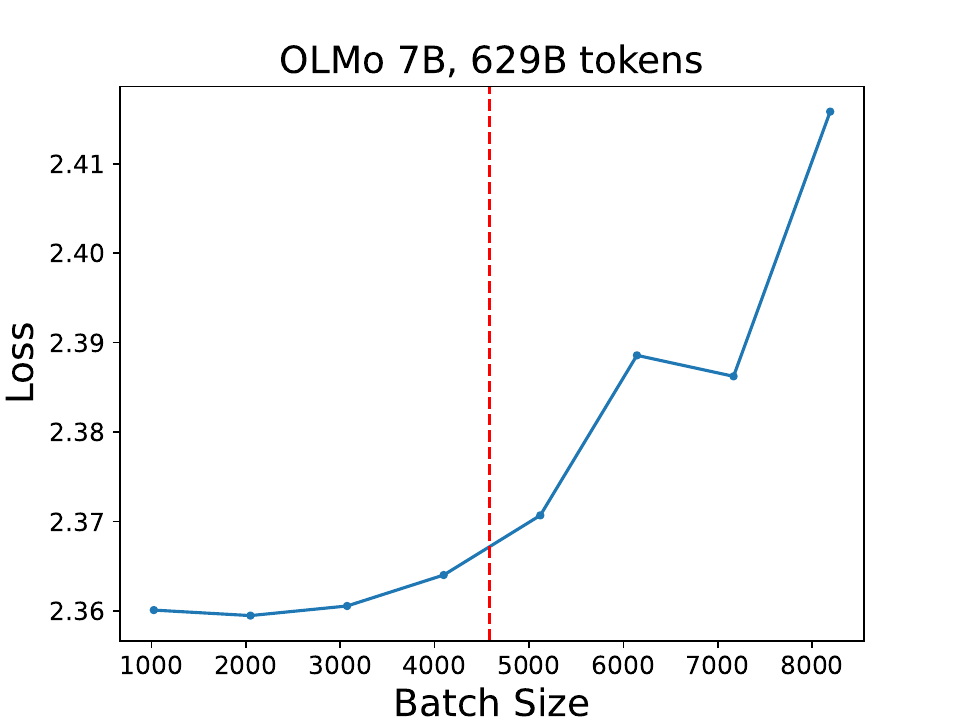}
    \includegraphics[width=0.32\linewidth]{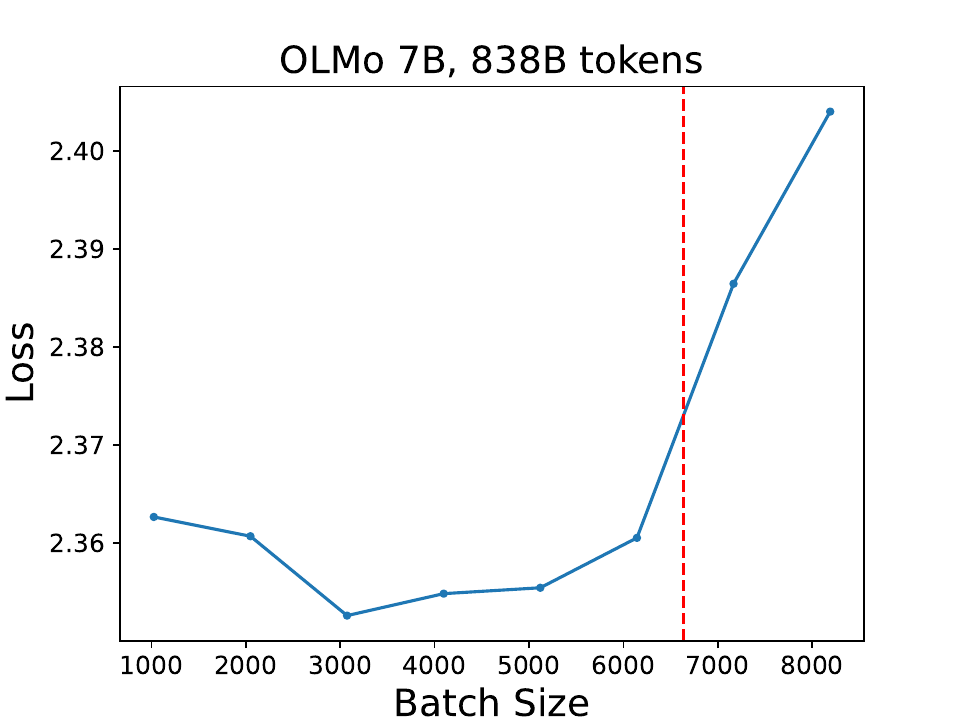}
    \includegraphics[width=0.32\linewidth]{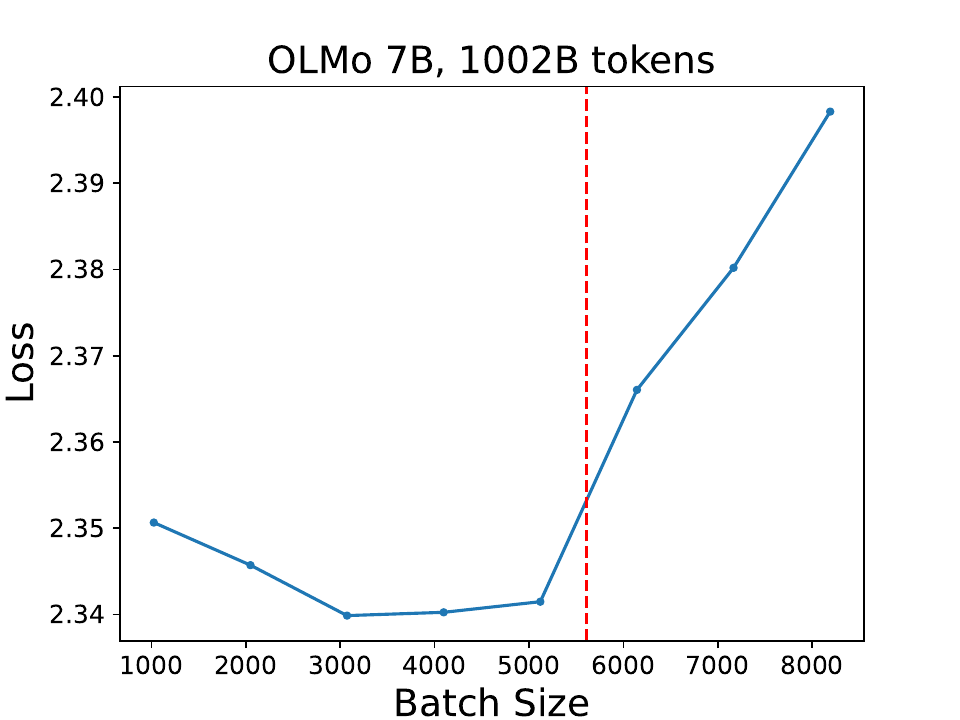}
    \includegraphics[width=0.32\linewidth]{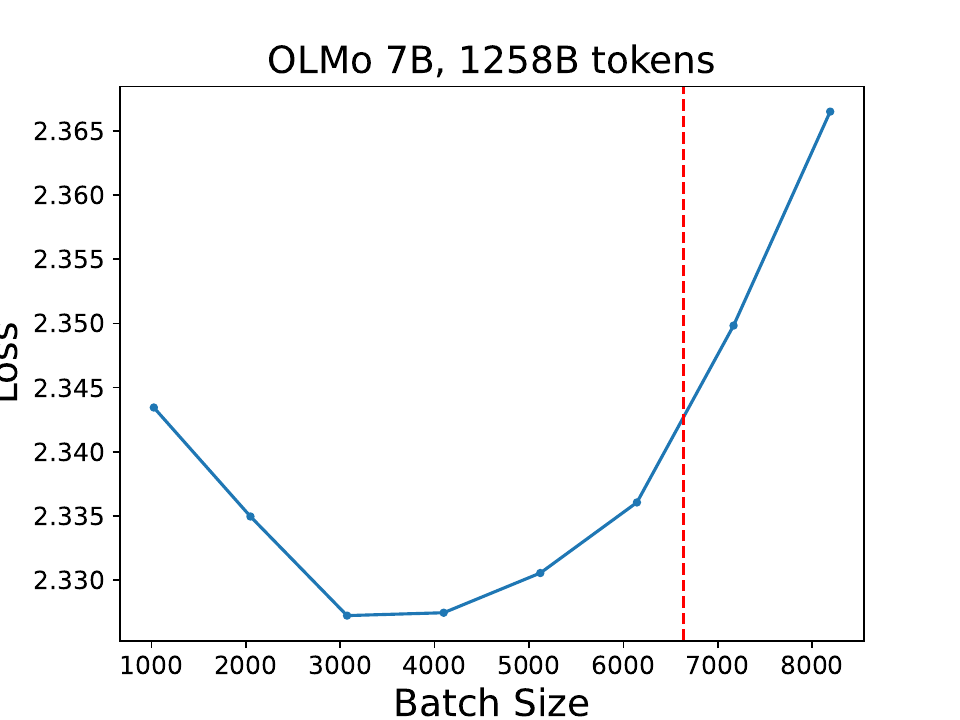}
    \includegraphics[width=0.32\linewidth]{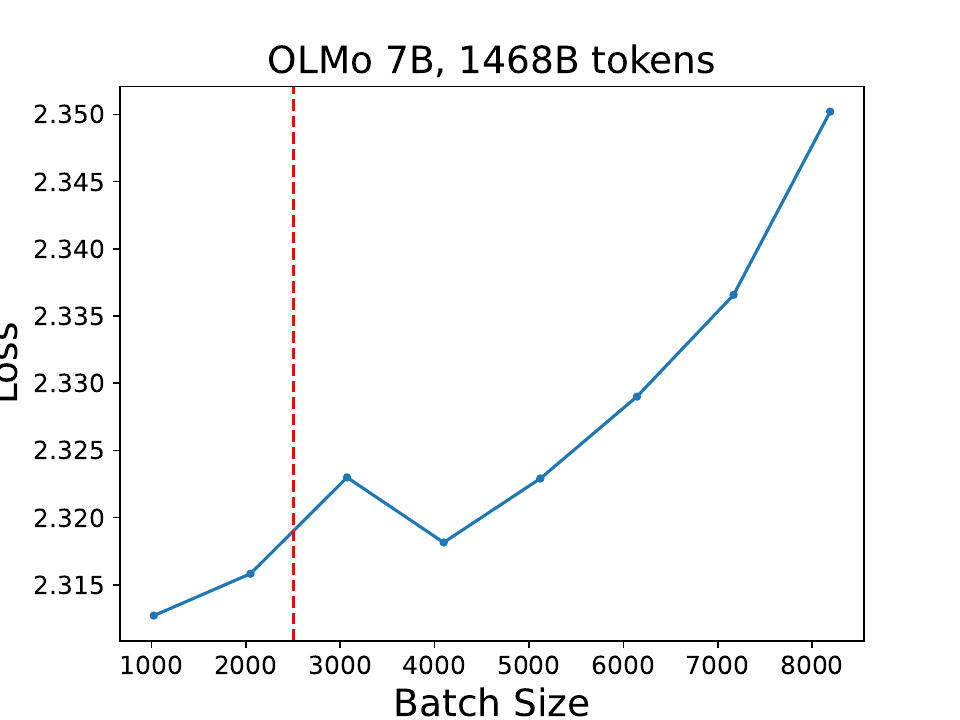}
    \includegraphics[width=0.32\linewidth]{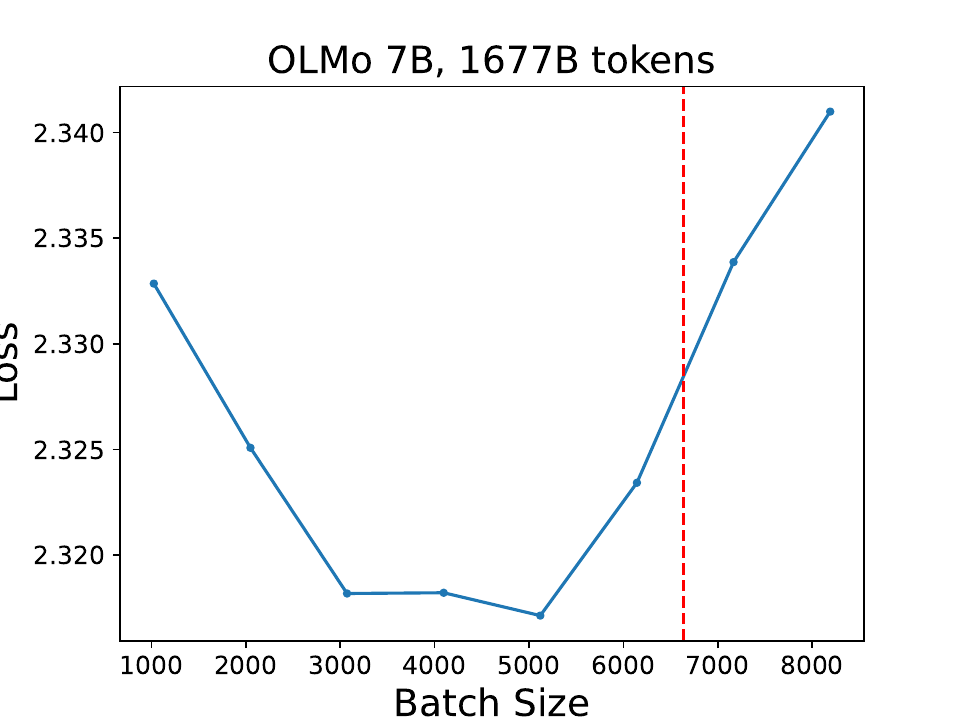}
    \includegraphics[width=0.32\linewidth]{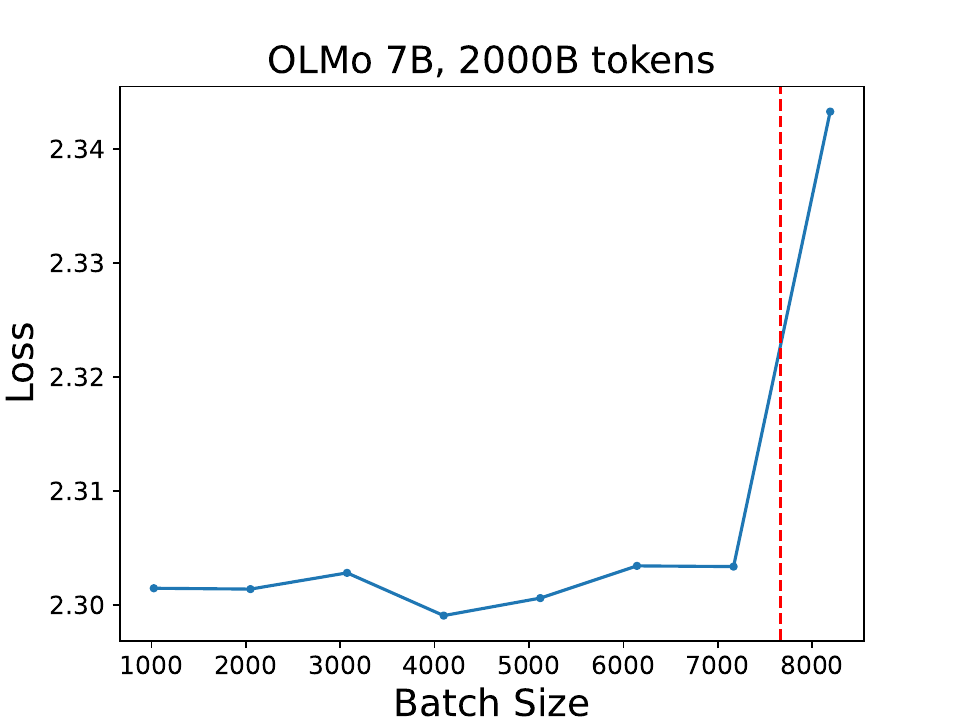}
    \label{fig:all-losses-7b}
\end{figure}

\section{Noise Scale Measurement Details} \label{app:noise-scale}

We use the gradient noise scale estimator proposed by \citet[Appendix A]{mccandlish2018empirical} to estimate the gradient noise scale.
The method estimates the gradient noise scale using gradient norms at two different batch sizes $\Bbig$ and $\Bsmall$ according to:
\begin{align*}
    \mathcal B_\mathrm{simple} &\approx \frac{\mathcal S}{\norm{\mathcal G}^2} , \textrm{where} \\
    \mathcal S &= \frac{\nGsmall - \nGbig}{1/\Bsmall - 1/\Bbig} \\
    \norm{\mathcal G}^2 &= \frac{\Bbig \nGbig - \Bsmall \nGsmall}{\Bbig - \Bsmall} . \\
\end{align*}
We use large batch size $\Bbig = 64$ and small batch size $\Bsmall = 1$.

It holds that $\mathbb{E} \left[ \mathcal S \right] = \mathrm{tr}(\Sigma)$ and $\mathbb{E} \left[ \norm{\mathcal G}^2 \right] = \norm{G}^2$. We thus average $\mathcal S$ and $\norm{\mathcal G}^2$ over 4096 batches reduce variance and then return their ratio as our estimate of the noise scale $\mathcal B_\mathrm{simple}$, using offline (i.e., unseen) data in each batch.

We estimate a confidence interval for $\mathcal B_\mathrm{simple}$ in two steps. First, we estimate 95\% confidence intervals for $\mathcal S$ and $\norm{\mathcal G}^2$, assuming the data are exponentially\footnote{For the exponential distribution, we use approximate confidence interval under ``Confidence Intervals'' here: \url{https://en.wikipedia.org/wiki/Exponential_distribution}.} and normally distributed, respectfully, based on manual inspection of their distributions (cf.~\Cref{fig:distributions}).
We denote these intervals $[a_{\mathcal S}, b_{\mathcal S}]$ and $[a_{\norm{\mathcal G}^2}, b_{\norm{\mathcal G}^2}]$, respectively.
We then define the confidence interval for $\mathcal B_\mathrm{simple}$ as follows:
\begin{equation*}
    \left[ \frac{a_{\mathcal S}}{b_{\norm{\mathcal G}^2}}, \frac{b_{\mathcal S}}{a_{\norm{\mathcal G}^2}} \right] .
\end{equation*}

If our estimates for $\mathcal S$ or $\norm{\mathcal G}^2$ (or their lower or upper bounds) come out negative, we consider them to be 0.

The checkpoints considered for OLMo 1B are steps 0, 10K, 20K, 40K, $\ldots$, 100K, 200K, $\ldots$ 400K.
For OLMo 7B, we use checkpoints at steps 0, 10K, $\ldots$, 40K, 60K, 70K, $\ldots$, 100K, 200K, $\ldots$ 400K.
The noise scale experiment for each checkpoint (for both the 1B and 7B models) was launched on a single GPU.

\begin{figure}
    \centering
    \includegraphics[width=0.48\linewidth]{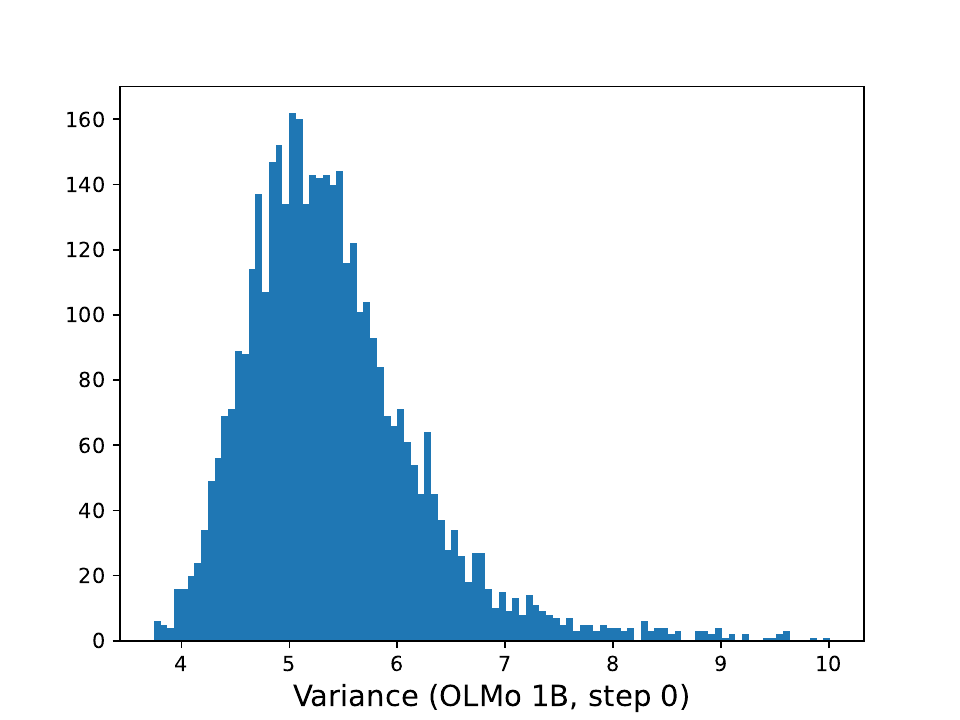}
    \includegraphics[width=0.48\linewidth]{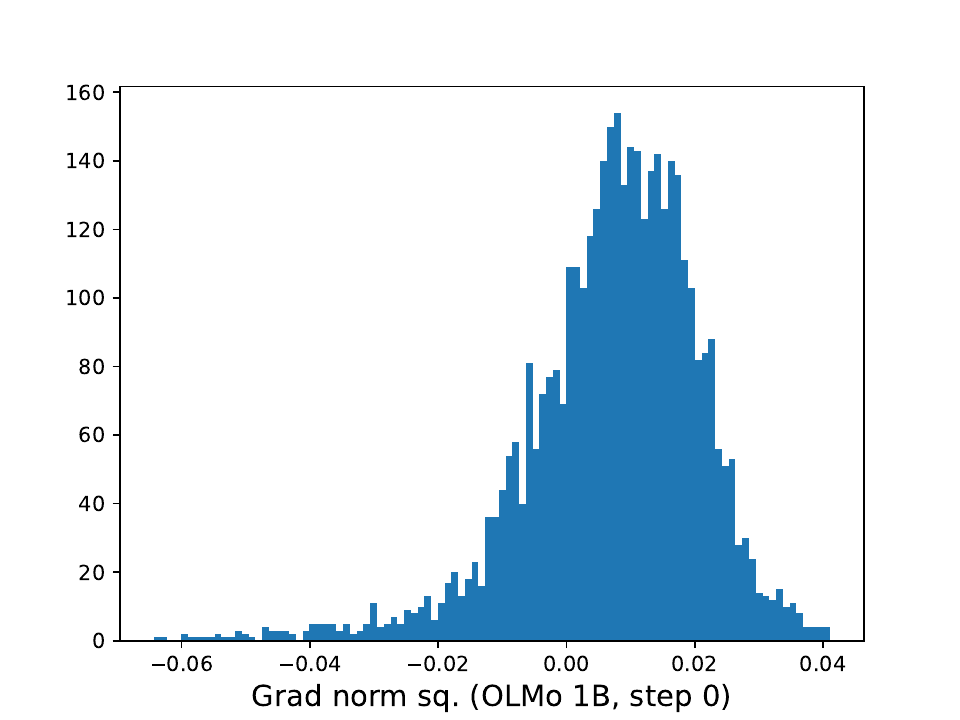}
    \caption{Representative histograms for $\mathcal S$ and $\norm{\mathcal G}^2$, showing data from the 1st to 99th percentiles. The distribution for $\mathcal S$ is positive, leading us to use an exponential distribution, while the fact that some samples of $\norm{\mathcal G}^2$ are negative motivates a normal distribution.}
    \label{fig:distributions}
\end{figure}

\section{License Information} \label{sec:other-details}

The OLMo models \citep{groeneveld-etal-2024-olmo,olmo20252olmo2furious} and pretraining code, which we use, are released under Apache-2.0 license.
C4 \citep{C4} is released under ODC-BY license.
The Pile \citep{Pile} is released under MIT license.

\section{Deriving CBS Scaling Laws: An Attempt} \label{sec:scaling-laws}

In this section, we explore whether our empirical fits for the critical batch size over training can be used to derive scaling laws for aggregate critical batch size that have been derived in prior work. These scaling laws assume we want to use a fixed batch size $B$ over training, and then train many different models to the same target loss. They then measure the critical $B^*$ up to which increases in batch size do not diminish token efficiency. The standard finding from such work is that CBS grows $\propto \sqrt{T}$, where $T$ is the total training budget in tokens. This is consistent with our finding that CBS increases over the course of training---moreover, we now seek to analyze whether this scaling law can be derived from our empirical measurements of local CBS. If so, this would provide converging evidence and a simpler method for fitting CBS scaling laws that only requires training a single model.

To begin, we assume that the goal of picking a fixed batch size $B$ is to minimize the L2 distance to the local CBS over the course of training.
It is not obvious that minimizing L2 distance is the right way to pick the fixed CBS: for instance, we might want to weight training at a batch size \emph{above} the local CBS more negatively than training below it.
Regardless, we will proceed for now under the assumption that this is the right perspective.
We also make the weaker assumption that $f(t) = 0$, in line with our empirical findings (\Cref{sec:our-method}).
It follows that the best batch size to train at (i.e., fixed CBS) is simply the average local CBS over training:

\begin{restatable}{proposition}{residuals} \label{thm:residuals}
    Let $f(t)$ be integrable with $f(0) = 0$ and define
    \begin{equation*}
        R_2 = \sqrt{\int_0^T \left( B - f(t) \right)^2 \mathrm{d}t} .
    \end{equation*}
    Then $R_2$ is minimized by $B^* = \frac{1}{T} \int_0^T f(t) \mathrm{d} t$.
\end{restatable}

\begin{proof}
    We can first simplify the expression for $(R_2)^2$:
    \begin{align*}
        (R_2)^2
        &= \int_0^T \left( B - f(t) \right)^2 \mathrm{d}t \\
        &= \int_0^T \left( B^2 - 2Bf(t) + f(t)^2 \right) \mathrm{d}t \\
        &= B^2T - \int_0^T \left( 2Bf(t) - f(t)^2 \right) \mathrm{d}t .
    \end{align*}
    Now, taking the derivative with respect to $B$, we get
    \begin{equation*}
        \frac{\mathrm{d}}{\mathrm{d}B} (R_2)^2
        = 2BT - 2\int_0^T f(t) \mathrm{d} t .
    \end{equation*}
    Note that the second derivative $2T$ is positive.
    Thus, setting the derivative to $0$ and solving for $B$, we conclude that the following value of $B$ minimizes $R_2$:
    \begin{equation*}
        B = \frac{1}{T} \int_0^T f(t) \mathrm{d} t . \qedhere
    \end{equation*}
\end{proof}

% \WM{This could be generalized to even norms larger than 2 by applying the binomial expansion. To handle the 1-norm, the analysis might look a bit different.}
% \SA{This doesn't hold for any other norms. I checked 1-norm manually just to be sure.}

Thus, under the assumptions we have made, if we are trying to pick a fixed batch size that best approximates the local CBS throughout training, we can simply pick the average CBS over training.
We can use \Cref{thm:residuals} to derive a scaling law for the fixed $B^*$ as a function of the final CBS or, equivalently, the total steps $T$.
We now consider various reasonable functional forms $f(t)$ for the CBS.

\subsection{Power Law CBS Scaling}

We first consider the prediction for the fixed CBS scaling law if the local CBS evolves as a power law.

\begin{restatable}[$B^*$ for power-law CBS]{proposition}{powerlaw}
    Let $f(t) = t^c$ for $c > 0$. Then the fixed CBS is
    \begin{equation*}
        B^* = \frac{1}{c+1} T^c .
    \end{equation*}
\end{restatable}

\begin{proof}
    Plug in and solve the integral:
    \begin{align*}
        B
        &= \frac{1}{T} \int_0^T t^c \mathrm{d}t \\
        &= \frac{1}{T} \cdot \bigg[\frac{t^{c+1}}{c+1} \bigg]^T_0 \\
        &= \frac{T^c}{c+1} . \qedhere
    \end{align*}
\end{proof}

In the case where $c=1/2$ (square root), $B^* = \frac{2}{3} B^*_T = \frac{2}{3} \sqrt{T}$, which derives the $\sqrt{T}$ scaling law proposed by prior work \citep{zhang2024cbs}.

\subsection{Logarithmic CBS Scaling}

\begin{restatable}[$B^*$ for log CBS]{proposition}{logcbs}
    Let $f(t) = \log (t + 1)$. Then the fixed CBS is
    \begin{equation*}
        B^* = \frac{T}{T+1} \log (T + 1) - 1 .
    \end{equation*}
\end{restatable}

\begin{proof}
    Plug in and solve the integral:
    \begin{align*}
        B
        &= \frac{1}{T} \int_0^T \log (t + 1) \mathrm{d}t \\
        &= \frac{1}{T} \cdot \bigg[\left((t + 1) \log (t+1) - t \right) \bigg]^T_0 \\
        &= \frac{T}{T + 1} \log (T + 1) - 1 . \qedhere
    \end{align*}
\end{proof}

Thus, for large $T$, the fixed CBS will scale as $B^* \approx \log T$.

\subsection{Discussion}

These results show that, if we are choosing the fixed batch size to minimize average distance to the CBS as it evolves over training, we should pick it, more or less, as a simple function that slightly discounts the final CBS. Specifically, if we believe that the local CBS grows as $\sqrt{T}$ during training, then this derives the $\sqrt{T}$ scaling law for $B^*$ proposed in prior work.

One limitation of this view is that the L2 residuals may not be the right way to measure closeness to the CBS. In particular, it may be worse to overestimate the CBS compared to underestimate, as training above the CBS (with a scaled up learning rate) can be unstable.
We thus do not read to much into this analysis, but view it as a potentially useful starting point for future empirical and theoretical that derives CBS scaling laws from the development of the local CBS over training.

\section{BPB Evaluation on Downstream Tasks} \label{app:bpb}

This section lists the datasets we used to compute BPB measures for downstream tasks.
For multiple-choice tasks, we use the Cloze/Completion Formulation (CF), and compute the BPB metric on the gold
answer.
For completion tasks, we simply compute BPB over the correct answer.
This approach was inspired by~\citet{Bhagia2024EstablishingTS}.
The selection of tasks follows the guidelines from~\citet{DataDecide}.

\begin{table*}[t]
  \centering

    \begin{small}
    \begin{tabular}{llllll}
    \toprule
    \bf{task} & \bf{split} & \bf{\# shots} & \bf{reference}\\
    \midrule
    ARC-Challenge & Test & 5 & \citep{clark2018think} \\
    ARC-Easy & Test & 5 & \citep{clark2018think}\\
    CommonsenseQA & Val & 5 & \citep{talmor-etal-2019-commonsenseqa} \\
    HellaSwag & Val & 5 & \citep{zellers-etal-2019-hellaswag}\\
    MMLU & Val and Test & 5 & \citep{hendryckstest2021}\\
    PIQA & Val & 5 & \citep{Bisk_Zellers_Le_bras_Gao_Choi_2020}\\
    Social IQa & Val & 5 & \citep{sap-etal-2019-social}\\
    WinoGrande & Val & 5 & \citep{Sakaguchi_Le_Bras_Bhagavatula_Choi_2020}\\
    GSM8K & Gold & 5 & \citep{cobbe2021training}\\
    Minerva & Gold & 0 & \citep{lewkowycz2022solving}\\
    Humaneval & Gold & 0 & \citep{chen2021evaluating}\\
    MBPP & Gold & 0 & \citep{austin2021program}\\
    Copycolors 10-way & & 0 & \citep{Wiegreffe2024AnswerAA}\\
    \bottomrule
    \end{tabular}
    \end{small}
\end{table*}

\clearpage
%%%%%%%%%%%%%%%%%%%%%%%%%%%%%%%%%%%%%%%%%%%%%%%%%%%%%%%%%%%%

\newpage
\section*{NeurIPS Paper Checklist}

\begin{enumerate}

\item {\bf Claims}
    \item[] Question: Do the main claims made in the abstract and introduction accurately reflect the paper's contributions and scope?
    \item[] Answer: \answerYes{} % Replace by \answerYes{}, \answerNo{}, or \answerNA{}.
    \item[] Justification: CBS measurements in \Cref{sec:our-method} and batch size warmup tested in \Cref{sec:batch-size-warmup}.

\item {\bf Limitations}
    \item[] Question: Does the paper discuss the limitations of the work performed by the authors?
    \item[] Answer: \answerYes{} % Replace by \answerYes{}, \answerNo{}, or \answerNA{}.
    \item[] Justification: Assumptions and limitations of our CBS method are discussed in \Cref{sec:our-method}.

\item {\bf Theory Assumptions and Proofs}
    \item[] Question: For each theoretical result, does the paper provide the full set of assumptions and a complete (and correct) proof?
    \item[] Answer: \answerYes{} % Replace by \answerYes{}, \answerNo{}, or \answerNA{}.
    \item[] Justification: Our (minor) theoretical results are fully justified in \Cref{sec:scaling-laws}.

    \item {\bf Experimental Result Reproducibility}
    \item[] Question: Does the paper fully disclose all the information needed to reproduce the main experimental results of the paper to the extent that it affects the main claims and/or conclusions of the paper (regardless of whether the code and data are provided or not)?
    \item[] Answer: \answerYes{} % Replace by \answerYes{}, \answerNo{}, or \answerNA{}.
    \item[] Justification: Hyperparameter information and other experimental choices are documented in \Cref{sec:cbs-measurement-details,app:noise-scale}.

\item {\bf Open access to data and code}
    \item[] Question: Does the paper provide open access to the data and code, with sufficient instructions to faithfully reproduce the main experimental results, as described in supplemental material?
    \item[] Answer: \answerNo{} % Replace by \answerYes{}, \answerNo{}, or \answerNA{}.
    \item[] Justification: If accepted, we will release code with the camera-ready version.

\item {\bf Experimental Setting/Details}
    \item[] Question: Does the paper specify all the training and test details (e.g., data splits, hyperparameters, how they were chosen, type of optimizer, etc.) necessary to understand the results?
    \item[] Answer: \answerYes{} % Replace by \answerYes{}, \answerNo{}, or \answerNA{}.
    \item[] Justification: Provided in \Cref{sec:cbs-measurement-details,app:noise-scale}.

\item {\bf Experiment Statistical Significance}
    \item[] Question: Does the paper report error bars suitably and correctly defined or other appropriate information about the statistical significance of the experiments?
    \item[] Answer: \answerYes{} % Replace by \answerYes{}, \answerNo{}, or \answerNA{}.
    \item[] Justification: \Cref{fig:measure-cbs} shows an interval representing lower and upper bounds on the CBS, as documented in \Cref{sec:our-method}.
    \Cref{fig:noise-scale} reports a confidence interval for the noise scale whose details are documented in \Cref{app:noise-scale}.

\item {\bf Experiments Compute Resources}
    \item[] Question: For each experiment, does the paper provide sufficient information on the computer resources (type of compute workers, memory, time of execution) needed to reproduce the experiments?
    \item[] Answer: \answerYes{} % Replace by \answerYes{}, \answerNo{}, or \answerNA{}.
    \item[] Justification: Provided in \Cref{sec:cbs-measurement-details,app:noise-scale}.
    
\item {\bf Code Of Ethics}
    \item[] Question: Does the research conducted in the paper conform, in every respect, with the NeurIPS Code of Ethics \url{https://neurips.cc/public/EthicsGuidelines}?
    \item[] Answer: \answerYes{} % Replace by \answerYes{}, \answerNo{}, or \answerNA{}.
    \item[] Justification: No explanation of special circumstances necessary.

\item {\bf Broader Impacts}
    \item[] Question: Does the paper discuss both potential positive societal impacts and negative societal impacts of the work performed?
    \item[] Answer: \answerNA{} % Replace by \answerYes{}, \answerNo{}, or \answerNA{}.
    \item[] Justification: This work is foundational research on the science of pretraining language models with no immediate impacts.
    
\item {\bf Safeguards}
    \item[] Question: Does the paper describe safeguards that have been put in place for responsible release of data or models that have a high risk for misuse (e.g., pretrained language models, image generators, or scraped datasets)?
    \item[] Answer: \answerNA{} % Replace by \answerYes{}, \answerNo{}, or \answerNA{}.
    \item[] Justification: No pretrained models or datasets released.

\item {\bf Licenses for existing assets}
    \item[] Question: Are the creators or original owners of assets (e.g., code, data, models), used in the paper, properly credited and are the license and terms of use explicitly mentioned and properly respected?
    \item[] Answer: \answerYes{} % Replace by \answerYes{}, \answerNo{}, or \answerNA{}.
    \item[] Justification: The licenses for OLMo, C4, and the Pile are acknowledged in \Cref{sec:other-details}.

\item {\bf New Assets}
    \item[] Question: Are new assets introduced in the paper well documented and is the documentation provided alongside the assets?
    \item[] Answer: \answerNA{} % Replace by \answerYes{}, \answerNo{}, or \answerNA{}.
    \item[] Justification: No new assets released.

\item {\bf Crowdsourcing and Research with Human Subjects}
    \item[] Question: For crowdsourcing experiments and research with human subjects, does the paper include the full text of instructions given to participants and screenshots, if applicable, as well as details about compensation (if any)? 
    \item[] Answer: \answerNA{} % Replace by \answerYes{}, \answerNo{}, or \answerNA{}.
    \item[] Justification: No human subjects involved.

\item {\bf Institutional Review Board (IRB) Approvals or Equivalent for Research with Human Subjects}
    \item[] Question: Does the paper describe potential risks incurred by study participants, whether such risks were disclosed to the subjects, and whether Institutional Review Board (IRB) approvals (or an equivalent approval/review based on the requirements of your country or institution) were obtained?
    \item[] Answer: \answerNA{} % Replace by \answerYes{}, \answerNo{}, or \answerNA{}.
    \item[] Justification: No human subjects involved.

\end{enumerate}

\end{document}